\documentclass[11pt]{article}

\usepackage{amsmath,amsfonts,bm,amssymb}

\def\eqref#1{equation~\ref{#1}}

\def\1{\bm{1}}

\def\va{{\bm{a}}}
\def\vb{{\bm{b}}}
\def\vc{{\bm{c}}}
\def\vd{{\bm{d}}}

\def\vh{{\bm{h}}}

\def\vr{{\bm{r}}}
\def\vs{{\bm{s}}}

\def\vu{{\bm{u}}}
\def\vv{{\bm{v}}}

\def\vx{{\bm{x}}}
\def\vy{{\bm{y}}}

\def\vlambda{{\bm{\lambda}}}

\def\evh{{h}}

\def\evv{{v}}

\def\mA{{\bm{A}}}
\def\mB{{\bm{B}}}

\def\mH{{\bm{H}}}

\def\mM{{\bm{M}}}

\def\mT{{\bm{T}}}

\def\mV{{\bm{V}}}
\def\mW{{\bm{W}}}

\def\mZ{{\bm{Z}}}

\def\mLambda{{\bm{\Lambda}}}

\DeclareMathAlphabet{\mathsfit}{\encodingdefault}{\sfdefault}{m}{sl}
\SetMathAlphabet{\mathsfit}{bold}{\encodingdefault}{\sfdefault}{bx}{n}
\newcommand{\tens}[1]{\bm{\mathsfit{#1}}}

\def\tT{{\tens{T}}}

\def\emLambda{{\Lambda}}
\def\emA{{A}}

\def\emH{{H}}

\def\emV{{V}}
\def\emW{{W}}

\newcommand{\etens}[1]{\mathsfit{#1}}

\def\etT{{\etens{T}}}

\newcommand{\R}{\mathbb{R}}

\newcommand{\KL}{D_{\mathrm{KL}}}

\DeclareMathOperator*{\argmin}{arg\,min}

\usepackage[margin=1in]{geometry}
\usepackage{lipsum}
\usepackage{amsfonts}
\usepackage{amsthm}

\usepackage{graphicx}
\usepackage{epstopdf}
\usepackage{algorithm}
\usepackage{algorithmic}
\usepackage{comment}
\usepackage{caption}
\usepackage{subcaption}
\usepackage{hyperref}
\usepackage[nameinlink,noabbrev]{cleveref}

\crefname{equation}{}{}
\Crefname{equation}{}{}

\ifpdf
  \DeclareGraphicsExtensions{.eps,.pdf,.png,.jpg}
\else
  \DeclareGraphicsExtensions{.eps}
\fi

\usepackage{makecell}
\usepackage{soul}

\usepackage{enumitem}
\setlist[enumerate]{leftmargin=.5in}
\setlist[itemize]{leftmargin=.5in}

\newcommand{\ti}{\tilde}

\newtheorem{theorem}{Theorem}[section]

\newtheorem{lemma}[theorem]{Lemma}
\newtheorem{proposition}{Proposition}[section]
\newtheorem{definition}{Definition}[section]

\newcommand{\articletitle}{An Efficient Newton Algorithm for Nonnegative Matrix Factorization with the Kullback-Leibler Divergence}

\title{\articletitle %
}

\author{Damien Lesens\thanks{ENS de Lyon. \texttt{damien.lesens@ens-lyon.fr}}
\and Jérémy E. Cohen\thanks{Univ. Lyon, INSA-Lyon, UCBL, UJM, CNRS, Inserm, CREATIS UMR 5220, U1294, F-69621, Lyon, France. \texttt{jeremy.cohen@cnrs.fr}}
\and Bora Uçar\thanks{CNRS and LIP (UMR5668, Universit\'e de Lyon - ENS de Lyon - UCBL - CNRS - Inria), 46 all\'ee d'Italie, Lyon, France. \texttt{bora.ucar@ens-lyon.fr}}.}

\usepackage{amsopn}
\DeclareMathOperator{\diag}{diag}

\begin{document}

\maketitle

\begin{abstract}
Nonnegative Matrix Factorization (NMF) is a fundamental tool in unsupervised learning, which approximates a nonnegative matrix by the product of two low-rank nonnegative factors.
The Kullback-Leibler (KL) divergence is best suited to measure the data to model discrepancy when the decomposed data sample follows a Poisson distribution, which is the case for count datasets such as term-document matrices or images.
Most KL-NMF algorithms in the literature minimize a separable majorant of the loss to find their next iterate. 
We argue that this method has reached its limits and propose to use instead the second-order Taylor expansion of the loss, leading to a Newton-type method.
We minimize this non-separable surrogate by proposing a generalization of the well-known HALS algorithm.
This yields an efficient KL-NMF algorithm which provably converges and which competes favorably with state-of-the-art algorithms on a large variety of datasets. 

\end{abstract}

\section{Introduction}
Nonnegative Matrix Factorization (NMF) decomposes a nonnegative matrix into the product of two nonnegative low-rank factors. Given a matrix $\mV \in \R_+^{M\times N}$ and a rank $R \le \min(M,N)$, the goal is to find $\mW \in \R_+^{M\times R}$ and $\mH \in \R_+^{R\times N}$ such that $\mV \approx \mW \mH$. 
Imposing nonnegativity on the factors differentiates NMF from the standard low-rank decomposition: nonnegativity makes the problem harder, but the resulting factors are interpretable. Consequently, NMF can be used in many domains such as chemometrics \cite{bro1996multiway,neymeyr2018set}, audio source separation \cite{fevotte2009nonnegative,smaragdis2003non}, hyperspectral imaging \cite{bioucas2012hyperspectral,cichocki2009applications} and representation learning \cite{hoyer2004nmf,modi2024model}. 

The fidelity of the model $\mW \mH$ to the data $\mV$ is measured entry-wise with a scalar loss function $d: \R_+ \times \R_+ \mapsto \R_+$. NMF can thus be formulated as the optimization problem
\begin{equation}
\label{eq:NMF}
\min_{\mW \in \R_+^{M \times R}, \mH \in \R_+^{R \times N}} \sum_{i=1}^M \sum_{j=1}^N d(\mV_{i,j}|(\mW\mH)_{i,j}).\tag{NMF}
\end{equation}
NMF with the squared Frobenius loss $\|\mV-\mW\mH\|_F^2$, that is $d(x|y) = (x-y)^2$, has been studied widely in the literature \cite{gillis2020nonnegative, guo2024rise}. 
It corresponds to the maximum likelihood estimator in the presence of additive independent and identically distributed (i.i.d.) Gaussian noise. %
However, in many applications the data is generated by a counting process. 
This is the case, for example, when documents are represented by vector of word counts \cite{fevotte2011algorithms,leskovec2020mining} or when images are interpreted as a photon counting process \cite{hasinoff2021photon}. 
In those applications, the error distribution is assumed to be Poisson, in other words 
$\mV \sim \mathcal{P}(\mW^*\mH^*)$
where $\mW^*$ and $\mH^*$ are some true underlying factors.
Here, the maximum likelihood estimator corresponds (up to a constant) to the Kullback-Leibler (KL) divergence. It is defined on scalars $x,y\ge0$ as $d_{KL}(x|y) = x\log(\frac{x}{y})-x+y$ with the convention that $0  \log(0) = 0$ and $a \log(0) = -\infty$ for $a > 0$. 
This work focuses on the \cref{eq:NMF} problem with the Kullback-Leibler divergence as a loss function, \cref{eq:NMF_KL} for short. More formally, we want to solve 
\begin{equation}\label{eq:NMF_KL}
\min_{\mW \in \R_+^{M \times R}, \mH \in \R_+^{R \times N}} \KL(\mV \Vert \mW \mH)\tag{NMF-KL}
\end{equation} 
\[
=\min_{\mW \in \R_+^{M \times R}, \mH \in \R_+^{R \times N}}\sum_{i=1}^M \sum_{j=1}^N (\emV_{i,j} \log \emV_{i,j} - \emV_{i,j} ) + \sum_{i=1}^M \sum_{j=1}^N ((\mW \mH)_{i,j} - \emV_{i,j} \log(\mW \mH)_{i,j} ).
\]

The \cref{eq:NMF_KL} problem is not convex in both $\mW$ and $\mH$. It is however %
convex in $\mW$ or $\mH$. Most algorithms for \cref{eq:NMF_KL} in the literature exploit biconvexity and update $\mW$ and $\mH$ alternatively. That is to update $\mH$ they try to decrease the convex function $\mH \rightarrow \KL(\mV \Vert \mW \mH)$ over $\R_+^{R \times N}$ without changing $\mW$. This problem can be further simplified by observing that it is separable with respect to each column of $\mH$ (or equivalently with respect to each row of $\mW$). We can then focus on the update of a single column, that is find a minimizer $\vh^*$ of
the vector valued function $\vh \rightarrow \KL( \vv \Vert \mW{\vh} )$ over $\R_+^R$.
 We call this problem NonNegative KL (NNKL), in reference to the NonNegative Least Squares (NNLS) problem.
Alternating algorithms solving \cref{eq:NMF_KL} try to decrease the \cref{eq:NNKL} loss for each column of the updated factor. 

In a seminal paper, Lee and Sung \cite{lee2000algorithms} proposed the Multiplicative Updates (MU) algorithm. 
It minimizes a tight separable upper-bound of the \cref{eq:NNKL} loss to find the next iterate,
and it is still the state of the art on many datasets \cite{hien_algorithms_2021}.
The MU algorithm is an instance of the more general Majorization-Minimization (MM) framework \cite{sun2016majorization}.  %
Another MM type algorithm for \cref{eq:NNKL} 
was proposed by Bauschke et al. \cite{bauschke_descent_2017}, 
and then used to solve \cref{eq:NMF_KL} \cite{hien_algorithms_2021}. 
 Our contribution in this manuscript is two-fold. %
 We first give some theoretical results on  Majorization-Minimization algorithms \cite{lee2000algorithms, bauschke_descent_2017}. 
 We show that the MU algorithm is optimal amongst algorithms that use a global separable majorant of the loss as a surrogate. These results motivate us to propose a novel algorithm that utilizes a non-separable surrogate. 

 The most obvious candidate for a non-separable surrogate is the second-order Taylor expansion of the loss. In the case of unconstrained optimization, Newton methods minimize exactly this approximation and are very popular in second-order optimization \cite{nesterov2018lectures}.  %
 Some works on the periphery of the NMF literature \cite{hansen_newton-based_2015,virtanen_active-set_2013,zdunek_nonnegative_2007} have tried to apply these tools to the \cref{eq:NNKL} problem. There are however several issues concerning Newton steps in this context. First, it is not clear how to handle the non-negativity constraints. Projected Newton steps \cite{zdunek_nonnegative_2007} or active-set like algorithms \cite{hansen_newton-based_2015,virtanen_active-set_2013} do not provide any guarantees in general. Second, computing a Newton descent direction can be costly, even in the unconstrained case. This is a critical issue if one wishes to compete with fast first-order algorithms. Last but not least, it is not clear which step size to choose in the absence of Lipschitz smoothness hypotheses. One can always resort to backtracking, but it can be a very resource intensive process.
The work of Tran-Dinh et al. \cite{tran2015composite} provides an answer on how to handle non-negativity constraints, as well as which step size to choose. They propose a second-order optimization framework for the constraint optimization of self-concordant functions. This class of functions is not a subset of Lipschitz smooth functions and contains the \cref{eq:NNKL} loss. Our work is in some sense a specialization of their framework to \cref{eq:NMF_KL}. However the specialization is not straightforward. 
 Our main contribution is to show that minimizing the second-order Taylor expansion under nonnegativity constraints can be done efficiently with a generalization of the HALS algorithm \cite{cichocki2009fast}. Our algorithm has complexity $O(MNR^2)$ per iteration, where $M\times N$ is the size of the data matrix $\mV$ and $R$ is the target rank. All the algorithms that use a separable surrogate have a cost per iteration of $O(MNR)$. Our experimental results show that our $O(MNR^2)$ method is competitive with state-of-the-art $O(MNR)$ algorithm. 
The rest of the paper is organized as follows. First, we give in \Cref{sec:existing_algo} an overview of existing algorithms for \cref{eq:NMF_KL}, with a focus on algorithms closely related to our method. Namely, we introduce Majorization-Minimization algorithms as well as some second-order algorithms. In \Cref{sec:res_UB} we give some results that motivate us to depart from those paradigms and propose in the rest of \Cref{sec:KL_HALS} our novel algorithm. Finally, we explain in \Cref{sec:exp} the implementation details of our method and present experimental results.%

\section{Background and related work}
\label{sec:existing_algo}
\subsection{Notation}
In the rest of the manuscript, $s$ denotes a scalar, $\vv$ a vector, $\mM$ a matrix and $\tT$ a tensor. 
For a matrix $\mM$, the $i$th row is denoted by $\mM_{:,i}$ and the $j$th column by $\mM_{:,j}$. By $\mM_{-i,:}$ we denote the matrix $\mM$ with the $i$th row removed. The same goes for removing columns: $\mM_{:,-j}$ denotes the matrix $\mM$ without its $j$th column. We use the same notations for tensor slices. For example for a $3$-order tensor $\tT$, the slice for which only the last index is fixed to be $k$ is denoted by $\tT_{:,:,k}$ (yielding a matrix). The symbol $\odot$ denotes the Hadamard (entry-wise) product. The division of two matrices or vectors $\frac{\vu}{\vv}$ is  also meant element-wise. %
The data matrix $\mV$ will always have size $M\times N$, while factors $\mW$ and $\mH$ have size $M\times R$ and $R \times N$ respectively. 
\subsection{Structure of the problem}
The problem \cref{eq:NMF_KL} is convex in $\mW$ and $\mH$ separately.  In consequence most algorithms in the literature alternate between the decrease of  $\mH \rightarrow \KL(\mV \Vert \mW \mH)$ over $\R_+^{R \times N}$ and $\mW \rightarrow \KL(\mV \Vert \mW \mH)$ over $\R_+^{M \times R}$, while maintaining the other factor fixed.

Some works \cite{marmin2023joint,vandecappelle2020second} have proposed algorithms that update both $\mW$ and $\mH$ at once. We will not discuss these algorithms nor explore their performance in this work. We will be focused on alternating algorithms as they account for the vast majority of the literature, and the algorithm we propose belongs to this family of solvers.

The factor update problem can be further simplified by observing that the minimized loss is separable with respect to each column of $\mH$ (or equivalently with respect to each row of $\mW$). In consequence, we can focus on the update of a single column and solve the \cref{eq:NNKL} problem, that is find $\vh^*$ such that 
\begin{equation}
    \label{eq:NNKL}
    \vh^* = \argmin_{\vh \in \R_+^R} \KL( \vv \Vert \mW{\vh} ).\tag{NNKL}
\end{equation}
The \cref{eq:NNKL} problem, which is also called Poisson regression, is difficult to solve with first-order optimization methods such as gradient descent. There are two main reasons for this.
First, gradients of \cref{eq:NNKL} are unbounded close to zero. Second, \cref{eq:NNKL} is not strongly convex away from the origin. Classical gradient descent results \cite{nesterov2013introductory} state that Lipschitz smoothness ensures the existence of a safe step-size. Furthermore, adding strong convexity guarantees a linear convergence rate of the iterates. Without those two hypothesis, we have to resort to more involved optimization techniques to get convergence guarantees and convergence speed in practice. %

\subsection{Majorization-minimization algorithms}
\label{sec:MMalgo}
Previous works have proposed using the Majorization-Minimization framework %
(MM) \cite{sun2016majorization} for solving \cref{eq:NNKL}. 
The idea of this method is to build a separable majorant of the loss, tight at the current iterate, which is minimized exactly to get the next iterate. The tight majoration ensures the decrease of the loss. 
The separability allows for a closed-form solution for each scalar entry. This formula can then be extended to the full factor matrix, giving a fast-to-compute closed-form update. For ease of presentation, we will only give formulas for the \cref{eq:NNKL} problem, in other words, for a single column of $\mH$. All formulas immediately extend to the full factor matrix $\mH$.

One of the most popular and performant MM algorithm is the Multiplicative Updates (MU) algorithm \cite{lee2000algorithms}, which was known earlier as the Richardson-Lucy iteration \cite{lucy1974iterative,richardson1972bayesian}. 
MU is also sometimes referred to as the Maximum Likelihood Expectation-Maximization (ML-EM) algorithm \cite{dempster1977maximum,fessler1994space}.
In what follows the current iterate will be noted $\tilde{\vh}$ and the variable on which we optimize $\vh$. The idea of the algorithm is to use Jensen's inequality to upper-bound the loss as
\begin{align*}
    \KL( \vv \Vert \mW \vh ) &= \sum_{i=1}^M \evv_{i} \log \evv_{i} - \evv_{i} + (\mW \vh)_i - \evv_{i} \log(\sum_{r=1}^R \emW_{i,r} \evh_{r})\\
    &= \sum_{i=1}^M v_i \log(v_i) - v_i + (\mW \vh)_i - v_i \log(\sum_{r=1}^R \alpha_{i,r}\frac{ W_{i,r} h_r}{\alpha_{i,r}}) \text{ with } \alpha_{i,r} = \frac{W_{i,r} \ti{h}_r}{(\mW \ti{\vh})_i} \\
    &\le \sum_{i=1}^M v_i \log(v_i) - v_i + (\mW \vh)_i - v_i \sum_{r=1}^R \alpha_{i,r} \log(\frac{ W_{i,r} h_r}{\alpha_{i,r}}) = \textit{UB}_{\textit{MU}}(\vh,\ti{\vh}).\\
\end{align*}
The separable upper-bound can then be minimized in $\vh$ to get the update
\begin{equation}
    \label{eq:update_MU}
    \vh = \underset{\vh \in \R_+^R}{\arg \min}\; \textit{UB}_{\textit{MU}}(\vh,\ti{\vh}) = \ti{\vh} \odot \frac{\mW^T \frac{\vv}{\mW \ti{\vh}}}{\mW^T \1_R}.
\end{equation}

Another type of MM algorithm was introduced by Bauschke et al. \cite{bauschke_descent_2017}. As an application of their optimization results on relative smoothness and Bregman divergences the authors prove the following result

\begin{proposition}[Lemma 7 \cite{bauschke_descent_2017}]
The function $\vh \mapsto \KL( \vv \Vert \mW \vh )$ is relatively smooth to Burg's entropy $\kappa(\vh) = -\sum_{r=1}^R \log h_r$ with relative smooth constant $L = ||v||_1$. This means that for any $\vh,\ti{\vh} \in \R^R$
\[
    \KL( \vv \Vert \mW \vh ) \le \KL( \vv \Vert \mW \ti{\vh} ) + \langle \nabla_\vh \KL( \vv \Vert \mW \ti{\vh} ), \vh - \ti{\vh}\rangle + ||\vv||_1\mathcal{B}_\kappa(\vh,\ti{\vh}),
\]
with $\mathcal{B}_\kappa(\vx,\vy) := \kappa(\vx) -\kappa(\vy) - \langle \nabla \kappa (\vy), \vx-\vy \rangle$ the Bregman divergence associated with $\kappa$.
\end{proposition}

Let us note
\[
    \textit{UB}_{\textit{Burg}}(\vh,\ti{\vh}) = \KL( \vv \Vert \mW \ti{\vh} ) + \langle \nabla_\vh \KL( \vv \Vert \mW \ti{\vh} ), \vh - \ti{\vh}\rangle + ||\vv||_1\mathcal{B}_\kappa(\vh,\ti{\vh}).
\]
Optimising on this tight separable upper-bound gives a closed-form multiplicative update similar to \cref{eq:update_MU} that we call MU Burg. 
This update was applied by Hien and Gillis \cite{hien_algorithms_2021} to solve the \cref{eq:NMF_KL} problem. They observe experimentally that the resulting NMF algorithm is consistently worse than other state-of-the-art algorithms. In particular, it is always outperformed by MU. We show in \Cref{prop:comp_UB} that its poor performance can be explained theoretically by a comparison of the upper-bound the two algorithms use.

\subsection{Second-order algorithms}
\label{sec:second_order_algo}

A collection of works \cite{pham2026second,hsieh2011fast,li2012fast, lin2020optimization,hien_algorithms_2021} have been using surrogate that incorporate second-order information. A straightforward way to build a second-order surrogate is to use the second-order Taylor expansion of the \cref{eq:NNKL} loss function. It approximates the objective function as
\[\KL( \vv \Vert \mW \vh ) \approx \KL( \vv \Vert \mW \ti{\vh} ) + \nabla_\vh \KL( \vv \Vert \mW \ti{\vh} )^T (\vh - \ti{\vh}) + \frac{1}{2} (\vh - \ti{\vh})^T \nabla^2_\vh \KL( \vv \Vert \mW \ti{\vh} )(\vh - \ti{\vh}).\]

In order to get a separable function from this formula, Pham et al. \cite{pham2026second} proposed to upper-bound the Hessian by a diagonal matrix of the form $\text{Diag}(\va)$ with $\va \in \R_+^R$, %
leading to a separable surrogate of the form
\[\KL( \vv \Vert \mW \vh ) \simeq  \KL( \vv \Vert \mW \ti{\vh} ) + \nabla_\vh \KL( \vv \Vert \mW \ti{\vh} )^T (\vh - \ti{\vh}) + \frac{1}{2} \sum_{r=1}^R a_r (h_r - \ti{h}_r)^2.\]
The surrogate is not a global upper-bound and depends on the choice of vector $\va$. 
Pham et al. discuss which $\va$ to choose in order to get good performance and convergence guarantees.

Another way to get a separable surrogate that exploits the second-order information is to use the second-order Taylor expansion with respect to a single entry in $\vh$. %
The minimization of this surrogate with the constraint $h_r\ge 0$ gives a closed-form formula for the next iterate. Several works \cite{hsieh2011fast,li2012fast, lin2020optimization} use this surrogate to update entries of $\mH$ in a cyclic way.
The most notable algorithm in this category is the Cyclic Coordinate Descent (CCD) algorithm \cite{hsieh2011fast}. It cycles through the entries of $\mH$ in a way that allows for a fast computation of first and second-order coefficients, making it a state-of-the-art algorithm for \cref{eq:NMF_KL}. 
Hien and Gillis \cite{hien_algorithms_2021} later proposed a variation of the CCD algorithm that they called Scalar Newton (SN). 
The SN algorithm is a slight modification of CCD that ensures that the loss decreases at each update. The proof of Hien and Gillis uses the self-concordance properties of the loss function and the convergence framework of Tran-Dinh et al. \cite{tran2015composite}.

A collection of works do not use separable surrogates, but rather employ classical Newton optimization tools (that can be found for example in \cite{nesterov2018lectures}). Virtanen et al. \cite{virtanen_active-set_2013} proposed an algorithm called ASNA that solves \cref{eq:NNKL} in the over-complete case (in other words $R\gg N$). Their algorithm is akin to an active set algorithm for NNLS and incorporates some Newton steps. In a similar fashion, Hansen et al. \cite{hansen_newton-based_2015} use Newton and quasi-Newton steps in an active-set-like algorithm for tensor decomposition with KL loss. Finally, Zdunek and Cichocki \cite{zdunek_nonnegative_2007} use projected Newton and quasi-Newton steps with damped Hessians. %
Finally, we would like to mention that not all works on \cref{eq:NNKL} rely on surrogate functions. The theory on duality \cite{rockafellar1997convex} has also been used to handle the linear term $\mW \vh$ inside the logarithm. The two main contribution on this method are a first-order primal dual algorithm proposed by Yanez and Bach \cite{yanez_primal-dual_2014} and an application of the ADMM algorithm to \cref{eq:NMF_KL} by Sun and Févotte \cite{sun_alternating_2014}.

\subsection{HALS algorithm}
The Hierarchical Alternating Least Square algorithm (HALS) \cite{cichocki2009fast} serves as the basis for our proposed algorithm. The HALS algorithm solves alternatively for $\mW$ and $\mH$ the optimization problem
\[\min_{\mH \in \R_+^{R \times N}} ||\mV - \mW \mH ||_F^2.\]
It minimize the loss with respect to each row of $\mH$ alternatively, in other words it solves
\[\min_{\vh \in \R_+^{N}} ||(\mV- \mW_{:,-r} \mH_{-r,:}) - \mW_{:,r} \vh^T ||_F^2,\]
where $\vh$ stands for the $r$th row of $\mH$. The update of the $r$th row has the 
closed form 
\[\mH_{r,:} = \max\left(\mH_{r,:}+\frac{ (\mW^T \mV)_{r,:} - (\mW^T \mW)_{r,:} \mH }{(\mW^T \mW)_{r,r}},0\right).\]
This formula can be applied multiple times for each row of $\mH$ efficiently by pre-computing the matrices $\mW^T \mV$ and $\mW \mW^T$. The total cost of the update of $\mH$ is $O(MNR + I R^2 N)$ where $I$ is the number of times each row of $\mH$ is updated.
\Cref{apdx:details_HALS} contains details and a pseudo-code for the HALS algorithm.

\section{KL-HALS: The proposed algorithm based on HALS}
\label{sec:KL_HALS}
Here, we present the proposed algorithm for \cref{eq:NMF_KL}. 
We first study two known upper-bounds used in MM algorithms (\Cref{sec:res_UB}) and we show the optimality of the MU upper-bound.
Based on this finding, we propose to use a
quadratic approximation of the loss function (\Cref{sec:so_approx}). 
The quadratic surrogate is minimized using a generalization of the HALS algorithm (\Cref{sec:gen_HALS}). 
Alternating the minimization for the two factors yields the proposed algorithm, which we call KL-HALS (\Cref{sec:kl_hals_algo}). Finally, by using the right step size our algorithm is proven to converge (\Cref{sec:cvg_results}).

\subsection{On separable upper-bounds}
\label{sec:res_UB}
In this section, we discuss the upper-bounds for MU and MU Burg introduced in \Cref{sec:MMalgo}. %
We first show that the upper-bound used by MU Burg is always larger than the one used by MU.

\begin{proposition}
    \label{prop:comp_UB}
    For all $\ti{\vh}$, $\vh$ non-negative
    \[\text{UB}_{\text{MU}}(\vh,\ti{\vh}) \le \text{UB}_{\text{Burg}}(\vh,\ti{\vh}).\]
\end{proposition}

\begin{proof}
    See \Cref{apdx:proof_comp_UB}.
\end{proof}

To the best of our knowledge this result was not known before. 
It explains why MU Burg consistently under-performs compared to MU \cite{hien_algorithms_2021}%
. Indeed, it is straightforward to see that the updates of MU will always produce a greater loss decrease than those of MU Burg. This rules out the MU Burg algorithm and makes one wonder whether it is possible to get a tighter separable upper-bound of the loss than MU. We give a negative answer to this question in the following proposition.

\begin{proposition}
\label{prop:MU_tight}
    Let $g_r : \R_+^R \mapsto \R_+$ be a collection of functions such that for all $\vh, \ti{\vh} \in \R_+^R$ it holds that
    \[\KL( \vv \Vert \mW \vh ) \le \sum_{r=1}^R g_r(h_r) \le \text{UB}_{\text{MU}}(\vh,\ti{\vh}).\]
    Then for all $\vh, \ti{\vh} \in \R_+^R$
    \[\sum_{r=1}^R g_r(h) = \text{UB}_{\text{MU}}(\vh,\ti{\vh}).\]
    In other words, $\text{UB}_{\text{MU}}$ is the tightest separable global majorant of the NNKL loss function.
\end{proposition}
\begin{proof}
    See \Cref{apdx:proof_MU_tight}
\end{proof}

This result motivates the search for non-separable and/or non-local surrogates of the loss function, as we know that there is no hope to find a global separable majorant of the loss tighter than MU.

\subsection{Second-order approximation}
\label{sec:so_approx}

We propose to use as a surrogate of the loss its second-order Taylor expansion with respect to all variables. 
To the best of our knowledge no other work on \cref{eq:NMF_KL} has proposed an algorithm that can minimize a non-separable second-order surrogate function %
efficiently. %
Let $\ti{\mW}$ and $\ti{\mH}$ denote the current factors. 
For the $j$th column of $\mH$, the second-order Taylor expansion of the loss writes as
\begin{align*}
\KL(\mV_{:,j} \Vert \ti{\mW} \mH_{:,j})
&\approx
\KL(\mV_{:,j} \Vert \ti{\mW} \ti{\mH}_{:,j})
+ \nabla_{\mH_{:,j}}
\KL(\mV_{:,j} \Vert \ti{\mW} \ti{\mH}_{:,j})^T
(\mH_{:,j}-\ti{\mH}_{:,j})
\\
& 
+ \frac{1}{2}
(\mH_{:,j}-\ti{\mH}_{:,j})^T
\nabla_{\mH_{:,j}}^2
\KL(\mV_{:,j} \Vert \ti{\mW} \ti{\mH}_{:,j})
(\mH_{:,j}-\ti{\mH}_{:,j})
\\
&=
C_j %
+ \bigl(\ti{\mW}^T(\mathbf{1}
- 2\tfrac{\mV_{:,j}}{\ti{\mW}\ti{\mH}_{:,j}})\bigr)^T
\mH_{:,j}
+ \frac{1}{2}
\mH_{:,j}^T
\ti{\mW}^T
\operatorname{diag}\!\left(
\frac{\mV_{:,j}}{(\ti{\mW}\ti{\mH}_{:,j})^2}
\right)
\ti{\mW}
\mH_{:,j},
\end{align*}
with $C_j$ a constant. Summing for all columns and removing constants, we obtain the surrogate optimization problem
\begin{equation}
\label{eq:surrogate_opti_pb}
    \min_{\mH\in\R_+^{R\times N}} \sum_{j=1}^N (\ti{\mW}^T(\1 - 2\frac{\mV_{:,j}}{\ti{\mW} \ti{\mH}_{:,j}}))^T\mH_{:,j} + \frac{1}{2}\mH_{:,j}^T \ti{\mW}^T \text{diag}(\frac{\mV_{:,j}}{(\ti{\mW} \ti{\mH}_{:,j})^2}) \ti{\mW} \mH_{:,j}.
\end{equation}

\subsection{Optimizing quadratics with nonnegativity constraints}
\label{sec:gen_HALS}
The HALS algorithm \cite{cichocki2009fast} minimizes efficiently a sum of quadratics with nonnegativity constraints:
\begin{equation}
\label{eq:hals_as_sum}
\min_{\mH \in \R_+^{R \times N}} ||\mV - \ti{\mW} \mH ||_F^2 = \text{constant} + \min_{\mH \in \R_+^{R \times N}} \sum_{j=1}^N - 2 (\ti{\mW}^T \mV_{:,j})^T \mH_{:,j} + \mH_{:,j}^T (\ti{\mW}^T \ti{\mW}) \mH_{:,j}.
\end{equation}
Observe that the only difference between the structures of problems \cref{eq:surrogate_opti_pb} and \cref{eq:hals_as_sum} is that in \cref{eq:hals_as_sum} 
the Hessians depends on the index $j$. Therefore, we propose a generalization of the HALS algorithm that minimizes a sum of quadratics with nonnegativity constraints and \textbf{different second-order terms}. We will use this algorithm to solve the %
problem \cref{eq:surrogate_opti_pb}. 

Let $\mA \in \R^{R\times N}$ be a matrix, and $\tT \in \R^{R \times R \times N}$ be a tensor with non-negative entries. Our algorithm solves %
\begin{equation}
    \label{eq:gen_HALS_pb}
    \min_{\mH \in \R_+^{R \times N}} \sum_{j=1}^N {\mA_{:,j}}^T\mH_{:,j} + \frac{1}{2} \mH_{:,j}^T \tT_{:,:,j} \mH_{:,j}. 
\end{equation}%
This problem is equivalent to minimizing a collection of multivariate quadratics with nonnegativity constraints. The matrix $\mA$ stores the first-order coefficients, and the tensor $\tT$ stores the Hessian of each quadratic form in its slices. We get \cref{eq:surrogate_opti_pb} by instantiating $\mA$ and $\tT$ as
\begin{equation}
\label{eq:AT_KL}
    \mA = \ti{\mW}^T(\1 - 2\frac{\mV}{\ti{\mW} \ti{\mH}}) \text{ and } \tT_{:,:,j} = \ti{\mW}^T \text{diag}(\frac{\mV_{:,j}}{(\ti{\mW} \ti{\mH}_{:,j})^2}) \ti{\mW} \text{ for all } j=1,\dots,N,
\end{equation}
and we get \cref{eq:hals_as_sum} with
$\mA =  - 2 \ti{\mW}^T \mV$ and $\tT_{:,:,j} = \ti{\mW}^T \ti{\mW} $ for all $ j=1,\dots,N$.

In a similar fashion to HALS, we optimize on the rows of $\mH$ alternatively. The update for the $r$th row has the following closed-form expression
\[ \text{\textbf{{GHALS:}}} \quad \emH_{r,j} = \max\left(\emH_{r,j} - \frac{\emA_{r,j} + \sum_{r'=1}^R \tT_{r,r',j} \emH_{r',j}}{\tT_{r,r,j}},0\right).\]
For details on how we obtain these updates see \Cref{apdx:details_general_HALS}.  We call our algorithm GHALS for \underline{G}eneralized HALS. 

The pseudo-code of our GHALS algorithm is presented in \Cref{alg:update_H} starting at Line \ref{line:GHALS}. 
Given the first-order coefficients in $\mA$ and second-order coefficients in $\tT$, it solves \cref{eq:gen_HALS_pb} for $\mH$. 
The cost of the algorithm is $O(IR^2 N)$, where $I$ is the number of times we cycle through each row of $\mH$. The entry-wise product of the two $R\times N$ matrices $\tT_{r,:,:}$ and  $\mH$ at Line \ref{line:update_H} costs $O(RN)$. Summing over columns of $\tT_{r,:,:} \odot \mH$ also costs $O(RN)$. The other operations at Line \ref{line:update_H} are performed in $O(N)$ cost. This line is repeated $I R$ times, yielding a total complexity of $O(IR^2N)$. %

\subsection{Algorithm}
\label{sec:kl_hals_algo}

\begin{algorithm}
    \caption{
    $\text{update-KL-HALS}(\mV,\ti{\mW},\ti{\mH},I)$
    }\label{alg:update_H}
    \hspace*{\algorithmicindent} \textbf{Inputs:} Data matrix $\mV$, current factors $\ti{\mW}$ and $\ti{\mH}$, number of GHALS iterations $I$ \\ %
    \hspace*{\algorithmicindent} \textbf{Output:} New factor $\mH$
    \begin{algorithmic}[1]
        \STATE $\blacktriangleright$\textbf{Compute $\mA$ and $\tT$ }
        \STATE $\mA := \ti{\mW}^T(\1 - 2\frac{\mV}{\ti{\mW} \ti{\mH}})$
        \STATE $\mB := (\text{vec}({\ti{\mW}_{i,:}}^T \ti{\mW}_{i,:}))_{i=1}^M =(\ti{\mW}[:,:,\texttt{None}] \odot \ti{\mW}[:,\texttt{None},:] ).\texttt{reshape}(M,R^2)$ \label{line:compute_B}
        \STATE $\tT := (\mB^T \frac{\mV}{(\ti{\mW} \ti{\mH})^2}).\texttt{reshape}(R,R,N)$ \label{line:comp_T}
        \STATE $\blacktriangleright$\textbf{Generalized HALS (GHALS) algorithm to solve \cref{eq:gen_HALS_pb}} \label{line:GHALS}
        \STATE $\mH := \ti{\mH}$
        \FOR{$k=1$ to $I$}
            \FOR{$r=1$ to $R$}
                \STATE $\mH_{r,:} := \max(\mH_{r,:} - (\mA_{r,:}+\text{sum\_columns} (\tT_{r,:,:} \odot \mH))/(\tT_{r,r,:}),0)$ \label{line:update_H}
            \ENDFOR
        \ENDFOR
        \STATE $\blacktriangleright$\textbf{Update $\mH$}
        \STATE $\Delta \mH := \mH-\ti{\mH}$ \COMMENT{Compute the descent direction\
        }
        \STATE Choose a step size per column $\vc \in \R_+^N$, see \Cref{sec:cvg_results} \label{line:KL-HALS:step_size_selection}
        \STATE $\mH := \ti{\mH} + \Delta \mH \text{diag}(\vc)$ \COMMENT{Update $\mH$ along the descent direction}
        \RETURN $\mH$
    \end{algorithmic}
\end{algorithm}

In this section, we sum up the main steps of our proposed algorithm. 
The pseudo-code for the update of $\mH$ is given in \Cref{alg:update_H}. 
We use the GHALS algorithm presented in the last section to solve the surrogate optimization problem \cref{eq:surrogate_opti_pb}. At each factor update we compute $\mA$ and $\tT$ as in \cref{eq:AT_KL}. 
The computation of $\tT$ can be performed efficiently with standard matrix operations  (see Lines \ref{line:compute_B} and \ref{line:comp_T} where we use \texttt{numpy} notations). %

The runtime of \Cref{alg:update_H} is $O(MNR^2 + INR^2)$. The first term $O(MNR^2)$ comes from computing the tensor $\mT$, which stores $N$ matrices of size $R\times R$. The $j$th matrix in $\tT$ is obtained by computing the weighted sum of the $M$ matrices $(\ti{\mW}_{i,:}^T \ti{\mW}_{i,:}))_{i=1}^M$ of size $R \times R$ with coefficients in $\frac{\mV_{:,j}}{(\ti{\mW} \ti{\mH}_{:,j})^2}$ (see \cref{eq:AT_KL}). This operation costs $O(MR^2)$ and is repeated $N$ times, yielding a total complexity of $O(MNR^2)$. We then apply the GHALS algorithm which costs $O(INR^2)$. Finally we update $\mH$ along the computed descent direction %
at Line \ref{line:KL-HALS:step_size_selection}. 
We use a step size that will guarantee convergence (see \Cref{sec:cvg_results}). This part of the algorithm has a negligible cost.

The proposed algorithm for the \cref{eq:NMF_KL} problem is presented in \Cref{alg:KL-HALS}. %
\Cref{alg:update_H} is applied $D$ times to $\mH$ before alternating with $\mW$.

\begin{algorithm}
    \caption{KL-HALS for solving \cref{eq:NMF_KL}}\label{alg:KL-HALS}
    \hspace*{\algorithmicindent} \textbf{Inputs:} Data matrix $\mV$, initial factors $\mW^0$ and $\mH^0$, number of outer iterations $K_{\max}$, number of inner iterations $D$, number of GHALS iterations $I$
    \begin{algorithmic}[1]
        \STATE $\mW,\mH := {\mW}^0,{\mH}^0$
        \FOR{$k=1$ to $K_{\max}$}
        \FOR{$n_{\text{inner}}=1$ to $D$}
            \STATE $\mH := \text{update-KL-HALS}(\mV,\mW,\mH,I)$
        \ENDFOR

        \FOR{$n_{\text{inner}}=1$ to $D$}
            \STATE $\mW := \text{update-KL-HALS}(\mV^T,\mH^T,\mW^T,I)^T$
        \ENDFOR
        \ENDFOR

        \RETURN $\mW,\mH$
    \end{algorithmic}
\end{algorithm}

\subsection{Convergence results}
\label{sec:cvg_results}

In this section we give some convergence guarantees on \Cref{alg:KL-HALS}. 
We show that our algorithm is a realization of the framework of Tran-Dinh et al. \cite{tran2015composite}. Their results guarantee that with the right step size at Line \ref{line:KL-HALS:step_size_selection} of 
\Cref{alg:update_H} each factor update will decrease the \cref{eq:NMF_KL} loss.

Let us start with a definition.
\begin{definition}[\textbf{Self-concordant functions} \cite{nesterov1994interior}]
    A convex function $f: \R^n \rightarrow \R$ is said to be self-concordant with parameter $M\ge 0$ if $|\varphi'''(t)| \le M \varphi''(t)^{3/2}$, where $\varphi(t):=f(\vs+t\vv)$ for all $t\in\R$, $\vx \in \text{dom}(f)$ and $\vv\in \R^n$ such that $\vx+t\vv \in \text{dom}(f)$. When $M=2$, the function is said to be standard self-concordant.
\end{definition}

The loss \cref{eq:NNKL} is self-concordant with parameter $M_{\text{NNKL}} = 2 \cdot \max \{ \frac{1}{\sqrt{v_i}} | v_i >0, i=1,\dots,M \}$ \cite{nesterov2013introductory}. This implies that $\vh \rightarrow M_{\text{NNKL}}^2/4 \cdot  \KL( \vv \Vert \mW{\vh} )$ %
is standard self-concordant. 
Tran-Dinh et al.~\cite{tran2015composite} propose three algorithms to solve the problem
\begin{equation}
    \min_{\vs \in \R^n} F(\vx) = f(\vx) + g(\vx), 
\end{equation}
with $f$ convex and standard self-concordant, $g:\R^n \rightarrow \R \cup \{+\infty\}$ proper, closed and convex. Typically, $f$ will be a loss and $g$ a non-smooth regularization term.
We get the \cref{eq:NNKL} problem by taking $f(\vh) := M_{\text{NNKL}}^2/4\cdot \KL(\vv|\mW\vh)$ and $g(\vh) := \delta_{\vu|\vu\ge0}(\vh) = 0$ if $h\ge0$ and $+\infty$ otherwise. %

Our proposed algorithm KL-HALS is an instance of the Proximal-Newton Method (Algorithm 1 in \cite{tran2015composite}). A Newton direction $\vd^k$ is found at iteration $k$ by first computing
\begin{equation}
    \label{eq:opti_quad_approx}
    \vs^k = \underset{\vx \in \R^n}{\argmin} \left( f(\vx^k)+\nabla f(\vx^k)^T(\vx-\vx^k) + \frac{1}{2}(\vx-\vx^k)^T \nabla^2 f(\vx^k)(\vx-\vx^k) + g(\vx)\right),
\end{equation}
where $\vx^k$ is the current iterate, and then by setting $\vd^k = \vs^k-\vx^k$. This is exactly the problem we solve for each column of $\mH$ with our GHALS algorithm. The next iterate is then a point found by taking a step $\alpha_k$ along the Newton direction, in other words $\vx^{k+1}:=\vx^k+\alpha_k \vd^k $. 

Tran-Dinh et al. prove that their algorithm follows an usual pattern in Newton-like algorithms, namely that there are two modes of convergence. If the current iterate is close to the optimal, we are in the basin of quadratic attraction: %
full Newton steps with $\alpha_k=1$ guarantee convergence with a quadratic rate. %
When outside of the basin of quadratic attraction, we have to perform damped Newton steps with an $\alpha_k$ less than $1$. 
The proximity to an optimal point is measured through the quantity $\lambda_k = \sqrt{\vd_k^T \nabla^2 f(\vx_k) \vd^k}$, called the Newton-decrement. 
The Proximal-Newton Method performs full Newton steps ($\alpha_k = 1$) when $\lambda_k$ is less than a constant $\sigma$, and damped Newton steps ($\alpha_k < 1$) otherwise (we refer to the original work \cite{tran2015composite} for more details).

We apply the factor update algorithm only a fixed number of times and we do not wait until convergence to a minimum. Hence, we cannot use the convergence rate results. Instead, we utilize the following proposition by Tran-Dinh et al. which shows that the loss \cref{eq:NNKL} decreases when damped Newton steps $\alpha_k = 1/(1+\lambda_k)$ are used. 

\begin{proposition}[parts of Theorem 6 \cite{tran2015composite}]
    \label{prop:descent_selfconc}
    If $\alpha_k = 1/(1+\lambda_k) \in (0,1]$, then the damped Newton step $\vx^{k+1}:=\vx^k+\alpha_k \vd^k$ satisfies
    \[F(\vx^{k+1}) \le F(\vx^k) - \omega(\lambda_k),\]
    where $\omega : \R \rightarrow \R_+$ is defined as $\omega(t):=t-\ln(1+t )$.
\end{proposition}
    A proof is given in \Cref{apdx:self-conc} for self-containment.

We propose two versions of \Cref{alg:update_H}:
\begin{itemize}
    \item \underline{KL-HALS-descent:} This version performs damped Newton steps with $\alpha_k = 1/(1+\lambda_k)$. It requires the computation of the Newton decrement $\lambda_k$. The cost of computing this quantity for all rows is $O(NR^2)$ ($O(R^2)$ repeated $N$ times). 
This cost is negligible compared to computing the descent direction and the tensor $\tT$. The use of the step size $1/(1+\lambda_k)$ %
in \Cref{alg:update_H} %
ensures that the loss decreases at each factor update. The loss is positive and decreases at each factor update so it converges to a limit point. 
    \item \underline{KL-HALS:} This version always performs full Newton steps with $\alpha_k=1$, even outside of the basin of quadratic convergence. It is a heuristic which makes the assumption that we will make more progress with larger step sizes, 
 even though we do not have any descent guarantee. We will see in \Cref{sec:exp} that this assumption is verified: on all tested instances KL-HALS converges faster than KL-HALS-descent.
\end{itemize}

KL-HALS-descent is proven to converge, but we do not have any guarantee on the point the algorithm converges to (e.g. local minimum, stationary point). Razaviyayn et al. \cite{razaviyayn2013unified} proposed the Block Successive Convex Approximation (BSCA) framework, which ensures the converge to a stationary point for factor update algorithms similar to KL-HALS. However, we could only make KL-HALS a true realization of BSCA by making some strong hypothesis on the problem. We thus leave the discussion on the links between KL-HALS and the BSCA framework to the supplementary material in \Cref{sec:bsca}.

Another aspect that can help second-order algorithms to converge in practice is a good initial point. We will discuss in \Cref{sec:scaling} how scaling the rows and columns of the initial model is a very efficient way to start close to the basin of attraction.

\section{Experimental results}
\label{sec:exp}

In this Section we investigate the experimental performance of our algorithms with respect to other state-of-art algorithms on a large collection of datasets. We first specify some implementation details in \Cref{sec:hyperparam,sec:code_sparse,sec:scaling}. We then list the algorithms we studied in \Cref{sec:exp_algo}. Finally, we present experimental results for real-life datasets (\Cref{sec:exp_real}) and synthetic datasets (\Cref{sec:exp_synth}).

\subsection{Choice of hyper-parameters}
\label{sec:hyperparam}
In \Cref{alg:KL-HALS} the value of hyper-parameters $I$ (the number of iterations of the proposed GHALS algorithm), and $D$ (the number of times a factor is updated before switching to the other one) needs to be specified. We find empirically that $D=1$ is the best choice: %
when applying \Cref{alg:KL-HALS} multiple times %
on the same factor the loss decrease is not worth the additional computing time. In other words, we make the most progress with the first update of a factor and we are better of switching to the other one after that. 
For the GHALS parameter $I$, we find that taking $I=5$ is enough, that is after cycling $5$ times through all ranks we have found a minimizer of the second-order Taylor expansion.

\subsection{Code for sparse matrices}
\label{sec:code_sparse}

Care must be taken to implement algorithms for large sparse datasets. 
Indeed, all state-of-the-art algorithms have to compute quantities $\frac{\mV}{\mW \mH}$ and/or $\frac{\mV}{(\mW\mH)^2}$ for the gradient and the Hessian of the cost function. 
Consider for example the MU \cite{lee2000algorithms} and AmSOM \cite{pham2026second} updates. 
\[
    \text{MU: }\mH = \mH \cdot \frac{\mW^T \frac{\mV}{\mW \mH}}{\mW^T \1_{M,N}} \text{ ; AmSOM: } \mH=\max(\mH+\gamma\frac{\mW^T \frac{\mV}{\mW\mH}- \mW^T \1_{M,N}}{(\mW^T \odot {\1}_{R,R} \mW^T)\frac{\mV}{(\mW\mH)^2}},\epsilon).
\]
These quantities also appear in our KL-HALS algorithm (see \Cref{alg:update_H}). %
For large and sparse matrices, 
computing the whole $\mW \mH$ matrix, which has $M\times N$ entries, is prohibitive. It is the case for example for document datasets~\cite{hien_algorithms_2021}, where large-scale experiments cannot be performed if a normal implementation is used.  

We overcome forming the product $\mW\mH$ by only computing the entries of $\mW \mH$ that correspond to non-zeros in $\mV$, whose number (that we note $\text{nnz}(\mV)$) can be orders of magnitude smaller than $MN$. 
This yields a sparse representation of $\frac{\mV}{\mW \mH}$ that can be handled efficiently. 
This way, our code only manipulates sparse matrices and the complexity of algorithms goes, for example, from $O(MNR)$ for dense matrices, to $O(\text{nnz}(\mV)R)$ for sparse matrices. 
This allows us to run an extensive evaluation of almost all state-of-the-art algorithm on a collection of large sparse matrices (see \Cref{sec:exp_sparse}). 
We used the \texttt{scipy} library \cite{2020SciPy-NMeth} and more specifically the \texttt{scipy.sparse} module to handle sparse matrices and sparse linear algebra. %

The only algorithms that we could not adapt to handle sparse matrices are the ones that use duality, namely FPA \cite{yanez_primal-dual_2014} and ADMM \cite{sun_alternating_2014}. 
These algorithms require the manipulation of a dual variable of size $M\times N$ that cannot be assumed to be sparse.  
In consequence, we have excluded these algorithms from our evaluation on sparse datasets.

\subsection{Initialization with optimal row and column scaling}
\label{sec:scaling}
Second-order methods need a good initial point to ensure convergence. Hien and Gillis \cite{hien_algorithms_2021} proposed to scale the initial model $\mW^0 \mH^0$ with a scalar coefficient such that the loss is minimal. The optimal coefficient admits a closed-form solution. 
Pham et al. \cite{pham2026second} proposed to scale the columns of $\mH$, as they observe that the optimization problem
\[\min_{\mLambda \in \R_+^{N\times N} \text{diagonal}} \KL(\mV \Vert \mW \mH \mLambda)\] has a closed-form solution $\mLambda$ defined for all $j \le N$ as
$\emLambda^*_{j,j} = \frac{\1^T \mV_{:,j}}{\1^T \mW \mH_{:,j}}.$
In other words, the scaling needs to make the column sums of $\mW \mH$ equal to the column sums of $\mV$.
The authors do not elaborate on the problem of finding an optimal scaling of both rows and columns, that is solving
\begin{equation}
\label{eq:opti_scaling}
    \min_{\mLambda_\mW \in \R_+^{M\times M}, \mLambda_\mH \in \R_+^{N\times N} \text{diagonal}} \KL(\mV \Vert \mLambda_\mW \mW \mH \mLambda_\mH).
\end{equation}
We  solve \cref{eq:opti_scaling} by alternatively scaling row and column sums of $\mW \mH$ until they match the ones of $\mV$ with Sinkhorn's algorithm \cite{sinkhorn1967concerning}.
We note that the optimal row/column scaling problem for \cref{eq:NMF_KL} was first mentioned in \cite{ho2008non}.

In all the experiments presented in \Cref{sec:exp} we use Sinkhorn's algorithm on the randomly generated initial factors to get a good starting point $\mW^0$,$\mH^0$. Applying Sinkhorn's algorithm for $J$ iterations costs $O(\text{nnz}(\mV)+J(M+N)R)$ for sparse matrices or $O(MN+J(M+N)R)$ for dense matrices. Taking $J=10$ is usually enough to get row and column sums equal to the right value up to machine precision. It thus stands as a very cheap initialization procedure that can be applied before running any \cref{eq:NMF_KL} algorithm. More details on the procedure are presented in \Cref{apdx:opti_scaling}.

\subsection{Algorithms}
\label{sec:exp_algo}

We compare our two algorithms KL-HALS and KL-HALS-descent with the state-of-the-art algorithms given below:
\begin{itemize}
    \item The MU algorithm \cite{lee2000algorithms}, see \Cref{sec:MMalgo};
    \item The First-order Primal-dual Algorithm (FPA) \cite{yanez_primal-dual_2014};
    \item The Cyclic-Coordinate Descent (CCD) algorithm \cite{hsieh2011fast}, see \Cref{sec:second_order_algo};
    \item The Scalar-Newton (SN) algorithm \cite{hien_algorithms_2021}, see \Cref{sec:second_order_algo};
    \item The AmSOM algorithm \cite{pham2026second}, see \Cref{sec:second_order_algo};
    \item The AMUSOM algorithm \cite{pham2026second}, which was proposed along AmSOM as an accelerated version of the MU algorithm;
    \item The HALS algorithm \cite{cichocki2009fast}, which is designed for the $L_2$-loss and serves as a baseline.
\end{itemize}
We exclude the MU Burg algorithm \cite{bauschke_descent_2017} from our experimental study as our theoretical results of \Cref{sec:res_UB} show that it will always under-perform the MU algorithm. %
We also exclude the ADMM algorithm \cite{sun_alternating_2014} as it was observed in \cite{hien_algorithms_2021} that it is unstable and does not converge on almost all tested datasets. %
The three works \cite{virtanen_active-set_2013,hansen_newton-based_2015,zdunek_nonnegative_2007} mentioned in \Cref{sec:second_order_algo} %
do not provide any convergence guarantee and the computational cost of their iterates is very expansive. For these reasons we will not study them here. %

We implemented all the above listed algorithms in \texttt{python}. For all algorithms we provide an implementation for dense matrices that uses the linear algebra subroutines of the \texttt{numpy} library. We also provide for all algorithms except FPA a code that handles large sparse matrices efficiently by using the \texttt{scipy.sparse} library and by making use of the remarks of \Cref{sec:code_sparse}. We use the \texttt{benchopt} library \cite{benchopt} for reproducibility and to make the handling of experimental results easier.
Our code is publicly available\footnote{\url{https://github.com/DamienLesens/benchmark_nmf_kl}}, but we do not provide all the datasets as some of them are under a license. However, we specify how to download those datasets and integrate them in the benchmark. For all datasets, we run all algorithms starting from the same 10 random initial points. %

\subsection{General comments on experimental results}
Overall our KL-HALS algorithm exhibits very good performance on most of the datasets we tested. It is amongst the bests if not the best algorithm, except on the MAPS and Verb dataset. %
For these datasets in particular we ran algorithms with increasing ranks, and the convergence of KL-HALS started to get slower for $R$ around $50$. %

KL-HALS-descent is consistently slower than KL-HALS. 
On a large portion of datasets it even struggles to make sufficient progress. This behavior is also observed for the SN algorithm which uses the same self-concordance properties. This suggest that the descent guarantees are an useful tool in theory but they do not translate to fast algorithms in practice. %
KL-HALS also consistently beats the CCD algorithm. This clearly indicates that taking into account the full second-order information is better than only using the second-order information with respect to single variables.

On the datasets we tested, the FPA algorithm is either amongst the worst algorithm or does not converge at all. We note that this algorithms does not come with any descent nor convergence guarantee (see Hien and Gillis \cite{hien_algorithms_2021} for a discussion).

We confirm the observations of previous studies \cite{hien_algorithms_2021,pham2026second} by observing that MU, AmSOM and AMUSOM are state-of-the-art algorithms that exhibit fast convergence on all datasets. %

\subsection{Real-life datasets} %
\label{sec:exp_real}
Real-life datasets are of three types: audio, image and document. The datasets that we use are listed in \Cref{tab:real_data_sets} along with their source, dimension, rank, and the time we let algorithms run on them. %

\subsubsection{Audio datasets}

Audio datasets consists of amplitude spectrograms obtained from two music recordings. %
For the first music excerpt we follow the methodology of Pham et al. \cite{pham2026second} and test algorithms on  the ``MAPS\_MUS-bach\_846\_AkPnBcht" excerpt taken from the MAPS music dataset \cite{emiya2010maps} for ranks $R \in \{2,11,23,45\}$. Results are presented in \Cref{fig:res_MAPS}.

\begin{figure}[h!]
    \centering
    \includegraphics[width=0.9\linewidth]{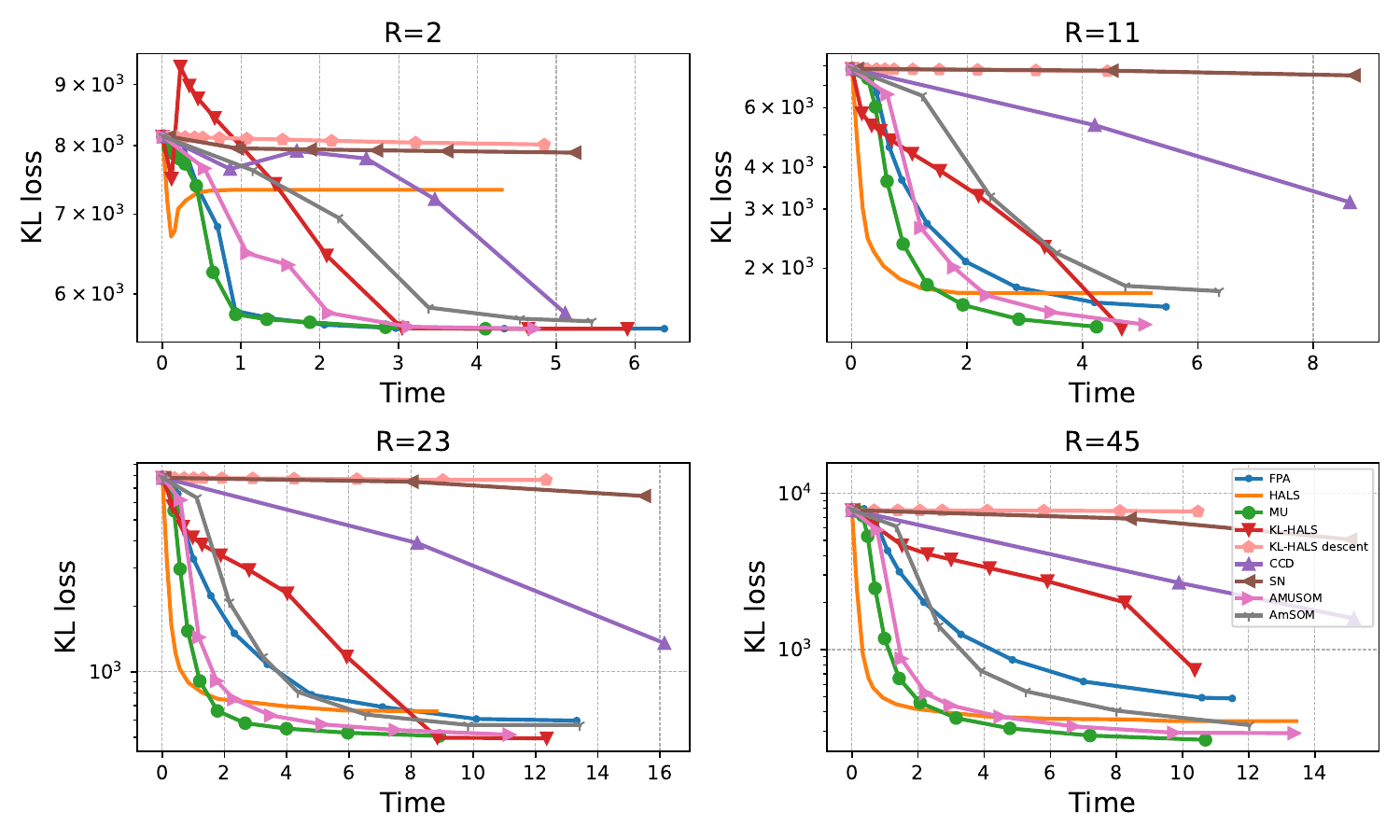}
    \caption{Median value of the KL loss on the amplitude spectrogram of an item of the MAPS dataset, for $R \in \{2,11,23,45\}$.}
    \label{fig:res_MAPS}
\end{figure}

We make the following remarks on the results of \Cref{fig:res_MAPS}:%
\begin{itemize}
    \item KL-HALS-descent and SN fail to descent sufficiently. It appears that for this dataset the step size that uses self-concordance properties is too conservative. %
    \item KL-HALS also struggles: for $R=2$ the loss does not always decrease. For all tested ranks it reaches the best loss, but at a slowest rate than other algorithms. For $R=45$, we stopped it before it could converge to the best loss. This is a consequence of the $R^2$ dependency of the iterates. %
    \item CCD converges too slowly to be competitive with other methods.
    \item MU exhibits the best behavior%
    : it always converges to the best loss the fastest.
    \item FPA follows MU closely for $R=2$ but its performance gets worse has $R$ grows. It remains a competitive algorithm on this dataset.
    \item HALS converges fast but to a higher loss than \cref{eq:NMF_KL} algorithms.
    \item AMUSOM stands out as the second best algorithm after MU.
    \item AmSOM struggles to converge but is amongst the best algorithms. %
\end{itemize}

The second excerpt was used for evaluation of NMF algorithms by Févotte et al. \cite{fevotte2009nonnegative}. It is a 108 seconds-long music excerpt from My Heart (Will Always Lead Me Back To You) recorded by Louis Armstrong and His Hot Five in the twenties. Results are presented in \Cref{fig:res_MyHeart} for $R=10$. For this dataset KL-HALS clearly outperforms all other algorithms by converging fast to the best loss value. Regarding other algorithms we can make the same remarks than for \Cref{fig:res_MAPS}.

\begin{figure}[h!]
    \centering
    \includegraphics[width=0.5\linewidth]{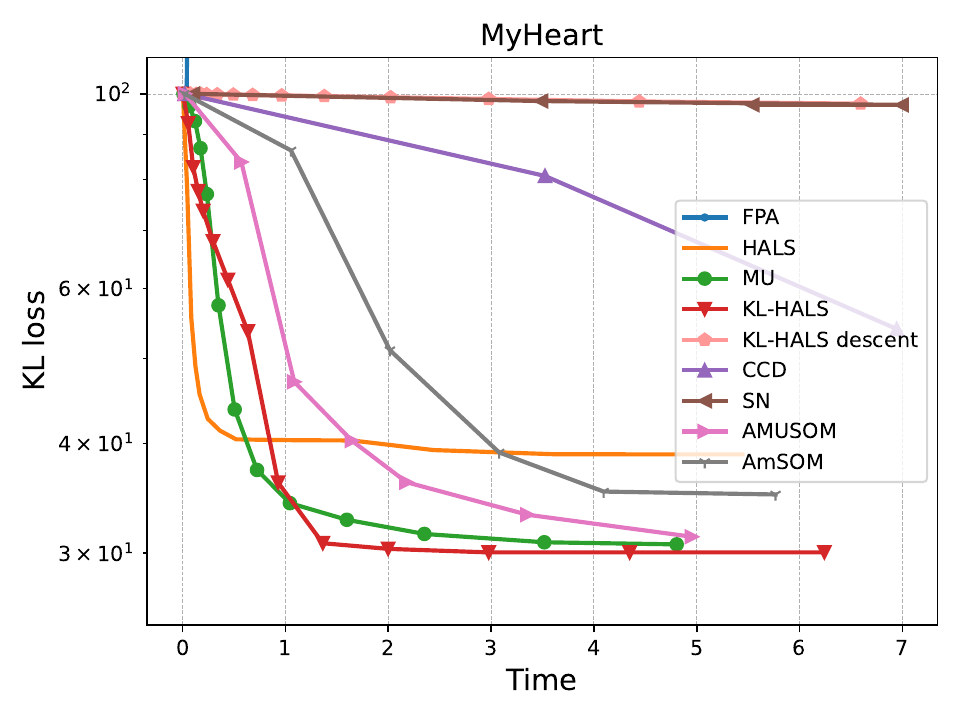}
    \caption{Median value of the KL loss on the amplitude spectrogram of an 108 seconds-long music excerpt from My Heart (Will Always Lead Me Back To You).}
    \label{fig:res_MyHeart}
\end{figure}

\subsubsection{Image datasets}

We run all algorithms on four image datasets listed in \Cref{tab:real_data_sets}. MIT-CBCL-faces \cite{sung96}, ORLfaces \cite{samaria1994parameterisation} and Frey \cite{frey_faces} are collections of face pictures of multiple people taken with different angles and/or lighting. Each picture is vectorized to form the column of the data matrix. The Urban dataset is an hyperspectral image widely used in the literature to showcase the performance of NMF methods on blind spectral unmixing tasks. Results for the four image datasets are presented in \Cref{fig:res_image}.

\begin{figure}[h!]
    \centering
    \includegraphics[width=0.9\linewidth]{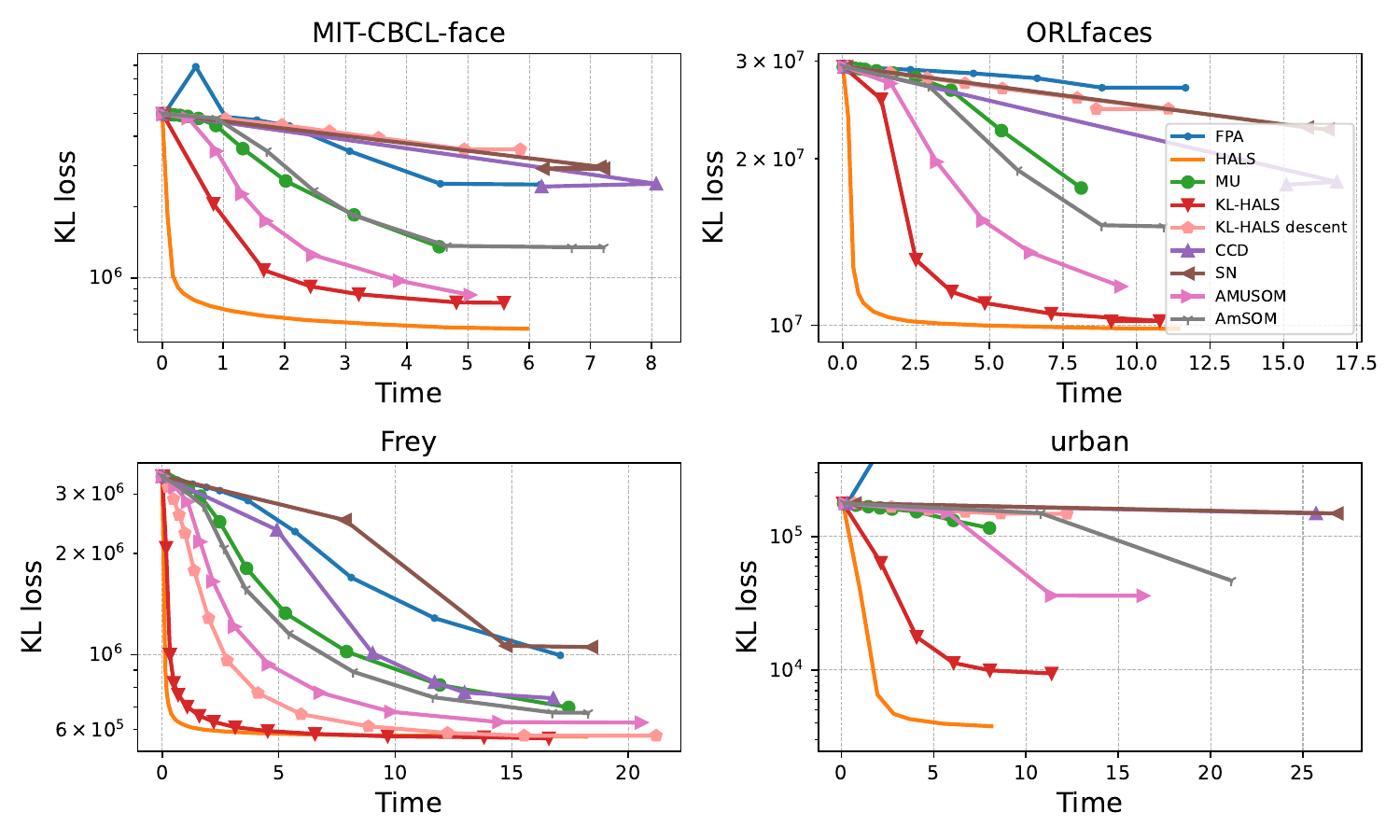}
    \caption{Median value of the KL loss on four image datasets.}
    \label{fig:res_image}
\end{figure}

On these datasets the convergence curves do not cross: algorithms that converge fast converge to a low loss value. Moreover, we can make the following comments on the results of \Cref{fig:res_image}.%
\begin{itemize}
    \item HALS is clearly the best algorithm on all tested datasets. This is interesting as it does not decreases the KL loss but the Frobenius norm. This may indicates that on this kind of datasets the Poisson and Gaussian noise models are equivalent, causing the KL and Frobenius losses to behave very similarly.%
    \item KL-HALS comes out as a clear second.
    \item KL-HALS-descent manages to be competitive only on the Frey dataset. On other datasets it does not descent sufficiently.
     \item AMUSOM is the third best algorithm.
     \item MU and AmSOM exhibit very close behavior and are not amongst the best algorithms.
     \item FPA, CCD and SN are the worst algorithms: they either struggle to converge, struggle to descent, or converge slower than other algorithms.
\end{itemize}
To summarize KL-HALS is the second best algorithm after HALS on image datasets. 

\subsubsection{Document datasets}
\label{sec:exp_sparse}
A document dataset consists of a collection of columns each representing a text document. The rows of the matrix represent words, and entries count how many times each word appears in each document. These datasets are large, with $M$ and $N$ in the order of $10^3$. They are also very sparse: on average only 1.86\% of entries were non-zero for the datasets we used. We thus utilized the remarks made in \Cref{sec:code_sparse} and our code takes into account the sparsity of $\mV$ to compute iterates efficiently. Like explained in \Cref{sec:code_sparse}, producing such a code for FPA is not possible due to its use of large and dense dual variables. In consequence we do not represent any result for the FPA algorithm in this Section.
Note that we removed zero rows corresponding to a word never appearing in any document of the collection. %
In consequence the size of matrices we used might not match the dimensions of the original dataset (see \Cref{tab:real_data_sets}).

Our first collection of results is obtained by running algorithms on all the CLUTO dataset \cite{cluto}. We will not present convergence curves for the 24 matrices in the collection. Instead we present in \Cref{fig:res_CLUTO} a performance plot. For a collection of relative errors (x axis) in the range $[10^{-6},1]$ the graph shows the percentage of time each algorithm has achieved this given relative error compared to the best algorithm. In consequence, the higher an algorithm is in the graph the better the loss it achieves.

\begin{figure}[h!]
    \centering
    \includegraphics[width=0.6\linewidth]{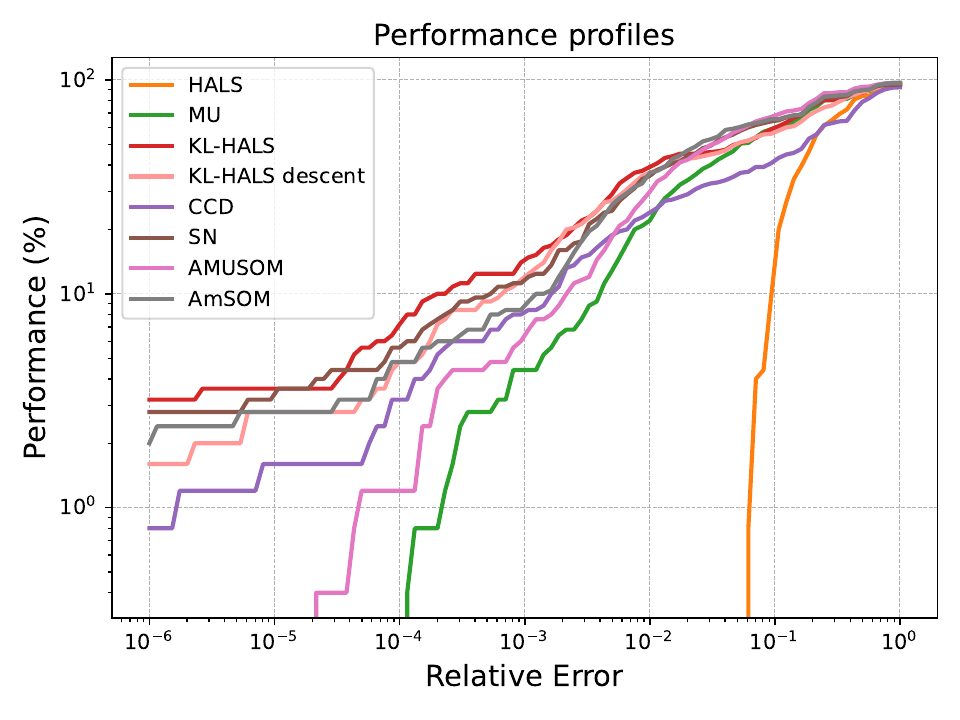}
    \caption{Performance of algorithms on the CLUTO dataset collection. %
    }%
    \label{fig:res_CLUTO}
\end{figure}

\Cref{fig:res_CLUTO} shows us that in terms of the best loss value reached:%
\begin{itemize}
    \item KL-HALS is the best algorithm: it has the best performance for all relative errors.%
    \item It is followed closely by SN and KL-HALS-descent. This is very different %
    from other datasets, on which the two algorithms struggle to make sufficient progress.
    \item AmSOM also exhibits good performance.
    \item CCD, AMUSOM and MU struggle to get close to the best loss values.
    \item HALS is clearly the worst algorithm: it can only get a 100\% performance when the allowed relative error gets close to 1. This suggest that it converges to a loss value that is much higher than other algorithms. We already observed this behavior on other datasets and it seems to be more pronounced here. Note that convergence speed of algorithms is not taken into account in performance plots, only the final value reached.
\end{itemize}

We also run algorithms on the Verb dataset, which was produced by \cite{tangherlini2016mommy}
and first used for NMF algorithm evaluation in \cite{lu2026exterior}. Results for ranks $R\in\{10,20,40,80\}$ are presented in \Cref{fig:res_Verb}. %

\begin{figure}[h!]
    \centering
    \includegraphics[width=0.9\linewidth]{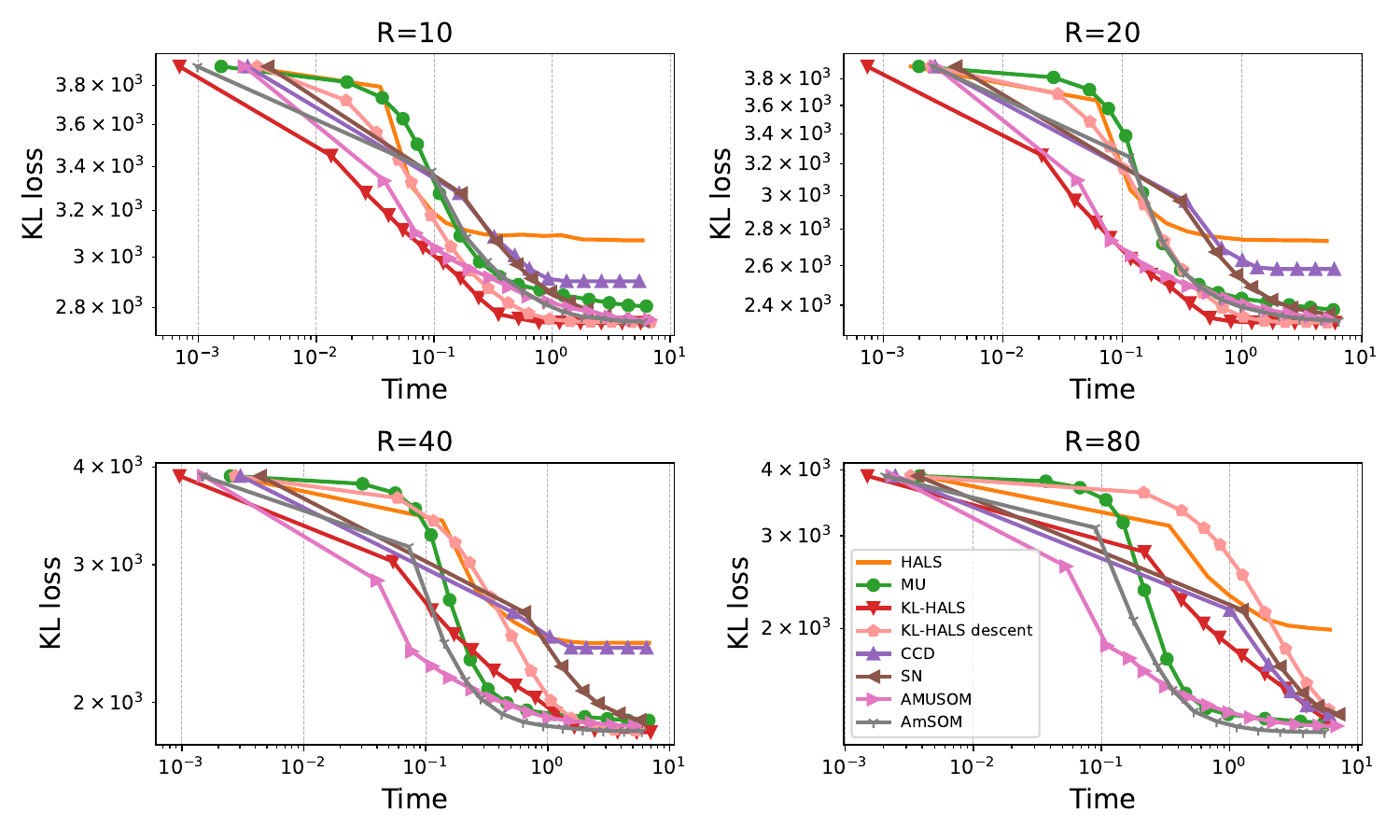}
    \caption{Median value of the KL loss on  the Verb dataset, for $R\in \{10,20,40,80\}$}
    \label{fig:res_Verb}
\end{figure}

We plot the results in log-log scale to better differentiate the position of convergence curves.
We interpret results of \Cref{fig:res_Verb} as follows:%
\begin{itemize}
    \item The results of KL-HALS degrade when the rank increases: it is clearly the best algorithm for $R=10$ but converges too slowly for $R=80$. This is very likely a consequence of the $R^2$ dependency in the cost of iterates.
    \item The convergence of KL-HALS-descent also degrades as the rank increases. This is expected as it has the same complexity than KL-HALS. We note than on this datasets it is competitive with other algorithms.
    \item HALS converges to a value much larger than other algorithms. The same behavior was observed in \Cref{fig:res_CLUTO}.
    \item CCD exhibits the same behavior than HALS except for $R=80$, where it converges to the best loss value.
    \item SN shows better convergence than CCD for all ranks.%
    \item MU starts off slowly but catches up with other algorithms toward the end.
    \item AmSOM and AMUSOM are amongst the best algorithm for all tested ranks. AMUSOM converges very fast for large ranks, while AmSOM achieves the lowest loss value
\end{itemize}

\subsection{Synthetic datasets}
\label{sec:exp_synth}
As it is a common practice, we experiment with low-rank synthetic and full-rank synthetic datasets  \cite{hien_algorithms_2021,pham2026second}. For each type of synthetic data specified below, we generate 10 random matrices. 
All algorithms will run on the same collection of matrices, and start from the same collection of initial point for each. 
In total, for each matrix generation process described below, each algorithm will run 100 times. We plot the median of all these runs.

\subsubsection{Low-rank synthetic datasets}

For matrix generation we follow the experimental set-up of Hien and Gillis \cite{hien_algorithms_2021}. Our underlying factors $\mW^*$ and $\mH^*$ are generated depending on a density parameter $l \in (0,1]$. The entries of the underlying factors are generated uniformly i.i.d. in $[0,1]$, and $1-l$ percent of entries are set to $0$. We then use either set $\mV=\mW^* \mH^*$ which corresponds to the noiseless case, or we generate each entry of $\mV$ according to a Poisson distribution of parameter $\mW^* \mH^*$. %
This models the noisy case, which is the main application of \cref{eq:NMF_KL}. We generate noisy and noiseless matrices for $l\in [1,0.9,0.3]$ and $(M,N,R) \in \{(200,200,10),(500,500,20)\}$. The results for $(N,M,R)=(200,200,10)$ are presented in \Cref{fig:res_lr_200} and for $(N,M,R)=(500,500,20)$ in \Cref{fig:res_lr_500}.

We leave the detailed analysis of these results to \Cref{sm:exp_synth_lr}: %
for $l\in\{1,0.9\}$ the ranking of algorithms is (by decreasing performance): KL-HALS, KL-HALS-descent, HALS, AMUSOM, AMSOM, CCD, SN,MU, FPA. For $l=0.3$ the ranking is: KL-HALS, AmSOM, AMUSOM, HALS, MU and then other algorithms (indistinguishable). We observe that KL-HALS is the only algorithm that converges to the best loss value for all datasets parameters.

\subsubsection{Full-rank synthetic datasets}

We again follow the experimental set-up of Hien and Gillis \cite{hien_algorithms_2021} by generating entries of $\mV$ with i.i.d. random Gaussian variables. This corresponds to the full-rank case as the data matrix does not present any low-rank structure a priori. We report results for full-rank synthetic $200\times200$ matrices enforcing $R=10$, as well as $500\times500$ matrices with $R=20$ in \Cref{fig:res_fr}. We leave the detailed discussion of these experimental results to \Cref{sm:exp_synth_fr} and simply observe that KL-HALS performs the best on both matrix sizes.

\begin{figure}[h!]
    \centering
    \includegraphics[width=0.9\linewidth]{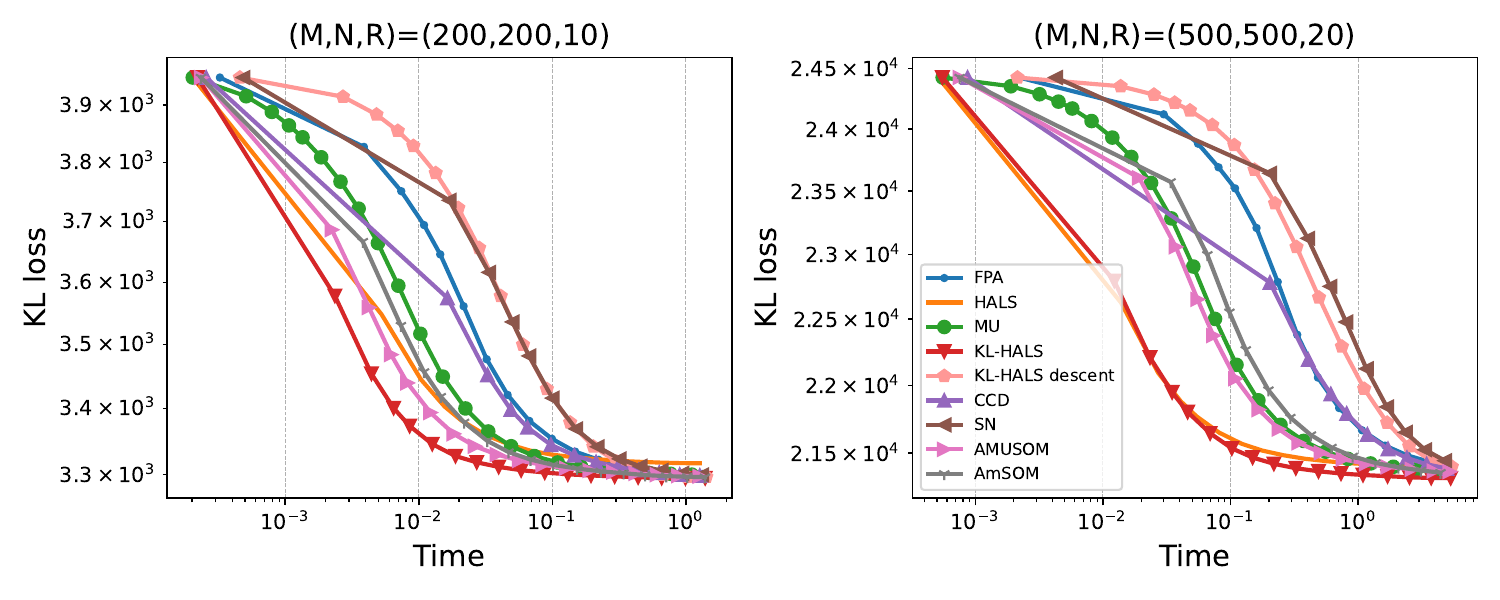}
    \caption{Median value of the KL loss on full-rank synthetic datasets: left - 200×200, right - 500×500.}
    \label{fig:res_fr}
\end{figure}

\section{Conclusion}

In this work, we present a novel algorithm for \cref{eq:NMF_KL} that utilizes the second-order development of the loss. 
We are the first work on \cref{eq:NMF_KL} proposing an algorithm that minimizes efficiently this non-separable surrogate under nonnegativity constraints. 
We invoke some results on self-concordant functions to show that there exists an easily computable step-size that makes the proposed algorithm converge. 
The large scale experiments we conducted show that our method out-performs state-of-the-art algorithms on most datasets. 
We believe that exploring non-separable surrogate is essential as there is no more margin for progress on separable upper-bounds: we prove that 
the global separable majorant used by the MU algorithm is the tightest possible. We also prove on the way the sub-optimality of the MU Burg algorithm. %

A promising research direction is to explore other kinds of non-separable surrogates, as this work shows they can lead to high-quality iterates.
We observed in experimental results that the additional time invested to compute those iterates is worthwhile, as the decrease in loss is tangible.  %
Also, some progress could be done on reducing the cost inferred by the computation and the minimization of the quadratic surrogate in our algorithm. %
Approximations of the Hessian are used in quasi-Newton methods to reduce the cost of iterates. We anticipate that similar approximation techniques could be used to make KL-HALS more competitive with first-order algorithms on datasets with high rank. 
We observed in our experiments that performing full Newton steps yields better results than shorter steps that are guaranteed to decrease the loss.
Recent results on extrapolated updates for the NMF problem \cite{ang2018accelerating,hien2025extrapolation} could motivate the use of steps even greater than one to accelerate convergence. It would also be interesting to explore whether the self-concordance descent guarantees could be strengthened in the case of \cref{eq:NNKL}, by either extending the size of the basin of quadratic convergence or by proving that there exists larger descending steps.%

\appendix

\bibliographystyle{IEEEtran}
\bibliography{references}

\appendix

\section{Details on HALS algorithm}

\label{apdx:details_HALS}

The HALS algorithm solves alternatively for each $r$
\[\min_{\vh \in \R_+^{N}} ||\mZ - \mW_{:,r} \vh^T ||_F^2, \]
where $\vh$ is the $r$th row of $\mH$ and $\mZ=\mV- \mW_{:,-r} \mH_{-r,:}$. This objective function is separable into $N$ scalar problems %
\[\min_{h_j\ge0} ||\mZ_{:,j}||_2^2 - 2\mW_{:,r}^T \mZ_{:,j} h_j + ||\mW_{:,r}||_2^2 h_j^2\]
for column indices $j \in \{1, \dots, N \}$. The minimum of a quadratic function under nonnegativity constraints is either the global minimum of the quadratic $h^*_j = \frac{\mW_{:,r}^T \mZ_{:,j}}{||\mW_{:,r}||_2^2}$ or $0$, hence
\[\argmin_{\vh \in \R_+^{N}} ||\mZ - \mW_{:,r} \vh^T ||_F^2 = \max \left(\frac{\mW_{:,r}^T \mZ}{||\mW_{:,r}||_2^2},0\right).\]

Pre-computing $\mW^T \mV$ and $\mW^T \mW$ is enough to perform the updates efficiently. Indeed, for the denominator, we have $||\mW_{:,r}||_2^2 = (\mW^T \mW)_{r,r}$. For the numerator
\begin{align*}
    \mW_{:,r}^T \mZ &= \mW_{:,r}^T (\mV- \mW_{:,-r} \mH_{-r,:})\\
&= \mW_{:,r}^T (\mV- \mW \mH + \mW_{:,r} \mH_{r,:})\\
&= (\mW^T \mV)_{:,r} - (\mW^T \mW)_{:,r} \mH + (\mW^T \mW)_{r,r}.
\end{align*}
Putting everything together gives the updates
\[\mH_{r,:} = \max \left(\mH_{r,:}+\frac{ (\mW^T \mV)_{r,:} - (\mW^T \mW)_{r,:} \mH }{(\mW^T \mW)_{r,r}},0\right).\]

The pseudo-code for the HALS algorithm factor update is given in \Cref{alg:update_H_HALS}.
\begin{algorithm}
    \caption{Algorithm for updating $\mH$ in HALS}\label{alg:update_H_HALS}
    \hspace*{\algorithmicindent} \textbf{Inputs:} Data matrix $\mV$, current factors $\ti{\mW}$ and $\ti{\mH}$ 
    \begin{algorithmic}[1]
        \STATE Compute $\ti{\mW}^T \mV$ and $\ti{\mW}^T \ti{\mW}$
        \STATE $\mH := \ti{\mH}$
        \FOR{$k=1$ to $I$}
            \FOR{$r=1$ to $R$}
                \STATE $\mH_{r,:} := \max(\mH_{r,:}+((\mW^T \mV)_{r,:} - (\mW^T \mW)_{r,:} \mH)/(\mW^T \mW)_{r,r},0)$
            \ENDFOR
        \ENDFOR
        \STATE \RETURN $\mH$
    \end{algorithmic}
\end{algorithm}

\section{Proof of \cref{prop:comp_UB}}

\label{apdx:proof_comp_UB}

    Let us first develop the MU Burg upper-bound. Recall that
    \[{\textit{UB}}_{\textit{Burg}}(\vh,\ti{\vh}) = \KL( \vv \Vert \mW \ti{\vh} ) + \langle \nabla_\vh \KL( \vv \Vert \mW \ti{\vh} ), \vh - \ti{\vh}\rangle + ||\vv||_1\mathcal{B}_\kappa(\vh,\ti{\vh})\]
    with $\nabla_\vh \KL( \vv \Vert \mW {\vh} ) = \mW^T (\1- \frac{\vv}{\mW \vh}),$
    and $\mathcal{B}_\kappa(\vx,\vy) := \sum_{r=1}^R ( \frac{x_r}{y_r} - \log(\frac{x_r}{y_r}) -1 ).$ %
    Hence, we obtain %
    \begin{multline*}{\textit{UB}}_{\textit{Burg}}(\vh,\ti{\vh}) = \sum_{i=1}^M v_i \log(v_i) - v_i + (\mW\tilde{\vh})_i - v_i \log((\mW\tilde{\vh})_i) \\+ \sum_{i=1}^M (\mW \vh)_i - (\mW \ti{\vh})_i - v_i\frac{ (\mW \vh)_i}{(\mW \ti{\vh})_i} + v_i + (\sum_{i=1}^M v_i)\sum_{r=1}^R \left( \frac{h_r}{\tilde{h_r}} - \log(\frac{h_r}{\tilde{h_r}}) -1\right),
    \end{multline*}
    which we rewrite as
    \begin{equation}
    \label{eq:MU_burg_dvlp}
        \sum_{i=1}^M v_i \log(v_i) + (\mW \vh)_i - v_i \log((\mW\tilde{\vh})_i) - v_i\frac{ (\mW \vh)_i}{(\mW \ti{\vh})_i} + v_i \sum_{r=1}^R \left( \frac{h_r}{\tilde{h_r}} - \log(\frac{h_r}{\tilde{h_r}}) -1\right).
    \end{equation}
    We now write the MU upper-bound in a similar fashion
    \begin{align*}
    \textit{UB}_{\textit{MU}}(\vh,\ti{\vh}) &= \sum_{i=1}^M v_i \log(v_i) - v_i + (\mW \vh)_i - v_i \sum_{r=1}^R \alpha_{i,r} \log(\frac{ W_{i,r} h_r}{\alpha_{i,r}}) \text{ with } \alpha_{i,r} = \frac{W_{i,r} \ti{h}_r}{(\mW \ti{\vh})_i}\\
    &=\sum_{i=1}^M v_i \log(v_i) - v_i + (\mW \vh)_i - v_i \sum_{r=1}^R \alpha_{i,r} (\log(\frac{ h_r}{\ti{h}_r}) +\log(\mW \vh)_i) \\
    &=\sum_{i=1}^M v_i \log(v_i)  + (\mW \vh)_i -v_i \log(\mW \ti{\vh})_i - v_i - v_i \sum_{r=1}^R \alpha_{i,r} \log(\frac{ h_r}{\ti{h}_r})\\
    &=\sum_{i=1}^M v_i \log(v_i)  + (\mW \vh)_i -v_i \log(\mW \ti{\vh})_i + v_i (-1- \sum_{r=1}^R \frac{W_{i,r} \ti{h}_r}{(\mW \ti{\vh})_i} \log(\frac{ h_r}{\ti{h}_r}))\\
    &=\sum_{i=1}^M v_i \log(v_i)  + (\mW \vh)_i -v_i \log(\mW \ti{\vh})_i + v_i \left( \sum_{r=1}^R \frac{W_{i,r} \ti{h}_r}{(\mW \ti{\vh})_i} (-\log(\frac{ h_r}{\ti{h}_r})-1)\right).
    \end{align*}
    The last equality comes from the fact that $\sum_r \alpha_{i,r} = 1$. Let us now focus on the last term of the sum. We add and subtract the term $\frac{ (\mW \vh)_i}{(\mW \ti{\vh})_i}$ to get
    \begin{align*}
    v_i \left( \sum_{r=1}^R \frac{W_{i,r} \ti{h}_r}{(\mW \ti{\vh})_i} (-1-\log(\frac{ h_r}{\ti{h}_r}))\right) &=  v_i \left( - \frac{ (\mW \vh)_i}{(\mW \ti{\vh})_i} + \frac{ (\mW \vh)_i}{(\mW \ti{\vh})_i}  + \sum_{r=1}^R \frac{W_{i,r} \ti{h}_r}{(\mW \ti{\vh})_i} (-\log(\frac{ h_r}{\ti{h}_r})-1)\right)\\
    &=  -v_i\frac{ (\mW \vh)_i}{(\mW \ti{\vh})_i} +v_i\left( \frac{ (\mW \vh)_i}{(\mW \ti{\vh})_i} + \sum_{r=1}^R \frac{W_{i,r} \ti{h}_r}{(\mW \ti{\vh})_i} ( -\log(\frac{ h_r}{\ti{h}_r})-1)\right)\\
    &=  -v_i\frac{ (\mW \vh)_i}{(\mW \ti{\vh})_i} +v_i\left(\sum_{r=1}^R \frac{W_{i,r} \ti{h}_r}{(\mW \ti{\vh})_i} ( \frac{h_r}{\ti{h}_r}-\log(\frac{ h_r}{\ti{h}_r})-1)\right).
    \end{align*}
    
    We get an expression very similar to the last two terms of \cref{eq:MU_burg_dvlp}. We thus know that the MU Burg upper-bound can be expressed as
     \[\KL( \vv \Vert \mW \ti{\vh} ) + \langle \nabla_\vh \KL( \vv \Vert \mW \ti{\vh} ), \vh - \ti{\vh}\rangle + \sum_{i=1}^N v_i \sum_{r=1}^R \left( \frac{h_r}{\tilde{h_r}} - \log(\frac{h_r}{\tilde{h_r}}) -1\right),\]
    while the upper-bound of MU can be written as
    \[\KL( \vv \Vert \mW \ti{\vh} ) + \langle \nabla_\vh \KL( \vv \Vert \mW \ti{\vh} ), \vh - \ti{\vh}\rangle + \sum_{i=1}^N v_i\sum_{r=1}^R \frac{W_{i,r} \ti{h}_r}{(\mW \ti{\vh})_i} \left( \frac{h_r}{\tilde{h_r}} - \log(\frac{h_r}{\tilde{h_r}}) -1\right).\]
    By concavity of the logarithm at $1$ we have that for all $x>0$, $x-\log(x)\ge 1$. Consequently for all $h_r>0$, $\frac{h_r}{\ti{h}_r} - \log(\frac{ h_r}{\ti{h}_r}) - 1 \ge 0$. We can then conclude by using the fact that for all $r$,  $\frac{W_{i,r} \ti{h}_r}{(\mW \ti{\vh})_i} \le 1$.

\section{Proof of \cref{prop:MU_tight}}

\label{apdx:proof_MU_tight}

    The upper-bound %
    $\textit{UB}_{\textit{MU}}(\vh,\ti{\vh})$ is a separable function in $\vh$. Let us note $\textit{UB}_{\textit{MU}}(\vh,\ti{\vh})) = \sum_{r=1}^R a_r(h_r)$. We assume that for all $\vh, \ti{\vh} \in \R_+^R$
    \[\KL( \vv \Vert \mW \vh ) - \textit{UB}_{\textit{MU}}(\vh,\ti{\vh}) \le \sum_{r=1}^R g_r(h_r) -a_r(h_r) \le 0,\]
    and we want to show that for all $\vh, \ti{\vh} \in \R_+^R$
    \[\sum_{r=1}^R g_r(h_r) -a_r(h_r) = 0.\]
We first observe that the upper-bound is tight on the semi-infinite line directed by $\ti{\vh}$.
\begin{lemma}
\label{lem:tightonline}
    For all $t \in \R_+$, $\KL( \vv \Vert \mW (t \ti{\vh}) ) = \textit{UB}_{\textit{MU}}(t\ti{\vh},\ti{\vh})$.
\end{lemma}
\begin{proof}
    We use notation from \Cref{sec:MMalgo}. The only ingredient of the upper-bound is the use of %
    Jensen's inequality on the weights $\alpha_{i,r}$ summing to $1$ and variables ${ W_{i,r} h_r}/{\alpha_{i,r}}$:
    \[  - v_i \log(\sum_{r=1}^R \alpha_{i,r}\frac{ W_{i,r} h_r}{\alpha_{i,r}}) 
\le  - v_i \sum_{r=1}^R \alpha_{i,r} \log(\frac{ W_{i,r} h_r}{\alpha_{i,r}}) \text{ with } \alpha_{i,r} = \frac{W_{i,r} \ti{h}_r}{(\mW \ti{\vh})_i}.\]

\[\text{for all } r,r' \in \{1,\dots,R\}, \quad \frac{ W_{i,r} h_r}{\alpha_{i,r}} = \frac{ W_{i,r'} h_{r'}}{\alpha_{i,r'}} \quad \]\[ \Leftrightarrow \quad \text{for all } r,r' \in \{1,\dots,R\}, \quad \frac{h_{r}}{\ti{h}_{r}} = \frac{h_{r'}}{\ti{h}_{r'}}.\]
In other words, the upper-bound is tight on the semi-infinite line $\{t \ti{\vh} \; | \; t \in \R_+\}$.
\end{proof}
This property of MU is enough to show that it is the best global separable upper-bound. We prove the the following Lemma:

\begin{lemma}
\label{lem:gap_is_0}
    Let $f_r : \R_+^R\rightarrow \R$ be a finite collection of functions such that %
     \[\text{for all } \vx \in \R_+^R, \quad \sum_r f_r(x_r)\le 0,\;\text{ and }\; \text{for all } t \in \R_+, \quad \sum_r f_r(t)= 0.\]
    Then all $f_r$ are constants on $\R_+$ and thus 
    \[\text{for all } \vx \in \R_+^R, \quad \sum_r f_r(x_r)= 0.\]
\end{lemma}

\begin{proof}
    Let $i$ be a function index, we show that $f_i$ is constant on $\R_+$. First, for all $t \in \R_+$

    \[f_i(t) = - \sum_{j\neq i}f_j(t).\]
    Then let $\vx \in \R_+^R$ be such that $x_i = y$ and for all $j\neq i$, $x_j=t$ then

    \[\sum_r f_r(x_r) \le 0 \Leftrightarrow f_i(y) + \sum_{j\neq i} f_j(t) \le 0 \Leftrightarrow f_i(y) - f_i(t) \le 0,\]

    so for all $y,t \in \R_+$, $f_i(y)\le f_i(t)$, hence $f_i$ is constant on $\R_+$. This concludes the proof of \Cref{lem:gap_is_0}
\end{proof}

Let us apply \Cref{lem:gap_is_0} to the functions $f_r$ defined as
\[\text{for all } r\in\{1,\dots,R\},\text{ for all } \vx\in\R_+^R, \quad f_r(x_r) = g_r(x_r \ti{h}_r) - a_r(x_r \ti{h}_r).\]
By hypothesis
\[\text{for all } \vx \in \R_+^R, \quad \sum_r f_r(x_r)\le 0,\]
and by \Cref{lem:tightonline}
\[\text{for all } t \in \R_+, \quad \sum_r f_r(t)= 0.\]
Hence
\[\text{for all } \vx \in \R_+^R, \quad \sum_r f_r(x_r)= \sum_r g_r(x_r \ti{h}_r) - a_r(x_r \ti{h}_r)= 0,\]
which is exactly what we want to prove.

\section{Details on generalized HALS algorithm}

\label{apdx:details_general_HALS}

The update formula is derived in a very similar fashion to the original HALS algorithm. We want to solve the following problem for a single row of $\mH$ at a time

\[ \min_{\mH \in \R_+^{R \times N}} \sum_{j=1}^N {\mA_{:,j}}^T\mH_{:,j} + \frac{1}{2} \mH_{:,j}^T \tT_{:,:,j} \mH_{:,j}.\]

For a fixed index $r$ let us make the dependency in the $k$th row of $\mH$ more explicit

\[\sum_{j=1}^N {\mA_{:,j}}^T\mH_{:,j} + \frac{1}{2} \mH_{:,j}^T \tT_{:,:,j} \mH_{:,j} = \sum_{j=1}^N \sum_{r=1}^R \emA_{r,j}\emH_{r,j} + \frac{1}{2} \sum_{r=1}^R \sum_{r'=1}^R\emH_{r,j} \etT_{r,r',j} \emH_{r',j}\]
\[= \sum_{j=1}^N C_j + (\emA_{k,j}+\sum_{r\neq k} \etT_{k,r,j} \emH_{r,j})\emH_{k,j} + \frac{1}{2}\etT_{k,k,j} \emH_{k,j}^2, \]
where $C_j$ is a constant.
The minimum of a quadratic function under nonnegativity constraints is either the global minimum of the quadratic or $0$, hence
\begin{align*}
\emH_{k,j} &= \max\left(-\frac{\mA_{k,j}+\sum_{r\neq k} \etT_{k,r,j} \emH_{r,j}}
{\etT_{k,k,j}},0\right)\\
&= \max\left(-\frac{\mA_{k,j}+\sum_{r=1}^R \etT_{k,r,j} \emH_{r,j} - \etT_{k,k,j} \emH_{k,j}}
{\etT_{k,k,j}},0\right)\\
&= \max\left(\emH_{k,j}-\frac{\mA_{k,j}+\sum_{r=1}^R \etT_{k,r,j} \emH_{r,j} }
{\etT_{k,k,j}},0\right).
\end{align*}

One can use fast vectorized operations in a library like \texttt{numpy} and compute the update as

\[\mH_{r,:} = \max(\mH_{r,:} - (\mA_{r,:}+\text{sum\_columns}(\tT_{r,:,:} \odot \mH))/(\tT_{r,r,:}),0).\]

\section{Proof of \cref{prop:descent_selfconc}}%
\label{apdx:self-conc}

We give here a reformulation of the original proof~\cite{tran2015composite} for self-containment. It is a simplification of their more general result. We do not claim any originality in this rewriting. %

By optimality of $\vs^k$ in \cref{eq:opti_quad_approx}
\[0 \in \nabla f(\vx^k) + \nabla^2 f(\vx^k)(\vs^k-\vx^k) + \partial g(\vs^k)\]
\[\Leftrightarrow - \nabla f(\vx^k) - \nabla^2 f(\vx^k)(\vs^k-\vx^k) \in \partial g(\vs^k)\]
\[\Rightarrow - \nabla f(\vx^k)^T (\vs^k -\vx^k) - (\vs^k-\vx^k)^T \nabla^2 f(\vx^k)(\vs^k-\vx^k) \in (\vs^k-\vx^k)^T\partial g(\vs^k).\]
Combining this with $g(\vx^k) - g(\vs^k) \ge \vv^T (\vx^k -\vs^k), \vv\in \partial g(\vs^k)$ (definition of subgradients of $g$ at $\vs^k$ and convexity of $g$) we get
\[g(\vs^k)-g(\vx^k) \le - \nabla f(\vx^k)^T (\vs^k -\vx^k) - \|\vs^k-\vx^k\|_{\vx^k}^2.\]%
The self-concordance property of $f$ implies the following local upper-bound \cite{nesterov2013introductory}: for all $\vx,\vy \in \text{dom}(f)$ such that $\|\vy-\vx\|_\vx<1$
\begin{equation}
\label{eq:UB_concordance}
    f(\vy) \le f(\vx) + \nabla f(\vx)^T (\vy-\vx) + \omega_*(\|\vy-\vx\|_\vx),
\end{equation}
where $\omega_*:[0,1]\rightarrow \R_+$ is defined as $\omega_*(t):=-t-\ln(1-t)$, and $\|\va\|_\vx = \sqrt{\va^T \nabla^2 f(\vx)\va}$ is the local norm around $\vx$ with respect to $f$.

Let us assume that we will choose $\alpha_k$ such that $\|\vx^{k+1} - \vx^k\|_{\vx^k}<1 \Leftrightarrow \alpha_k \|\vs^k - \vx^k\|_{\vx^k} < 1$. Note that this is not a condition on $\vs^k$, but on $\alpha_k$, as $\vs^k$ is fixed to be a solution of \cref{eq:opti_quad_approx}. %
We can use the self-concordance upper-bound with $\vx^{k+1}$ and $\vx^k$ to upper-bound $F(\vx^{k+1})$ as

\[F(\vx^{k+1}) \le f(\vx^k) + \alpha_k \nabla f(\vx^k)^T (\vs^k-\vx^k) + \omega_* (\alpha_k \|\vs^{k} - \vx^k\|_{\vx^k}) + g(\vx^{k+1})\]
\[\le f(\vx^k) + \alpha_k \nabla f(\vx^k)^T (\vs^k-\vx^k) + \omega_* (\alpha_k \|\vs^{k} - \vx^k\|_{\vx^k}) + g(\vx^k) + \alpha_k(g(\vs^k)-g(\vx^k))\]
\[\le F(\vx^k) + \alpha_k \nabla f(\vx^k)^T (\vs^k-\vx^k) + \omega_* (\alpha_k \|\vs^{k} - \vx^k\|_{\vx^k}) + \alpha_k (- \nabla f(\vx^k)^T (\vs^k -\vx^k) - \|\vs^k-\vx^k\|_{\vx^k}^2)\]
\[=  F(\vx^k) + \omega_* (\alpha_k \|\vs^{k} - \vx^k\|_{\vx^k}) - \alpha_k \|\vs^k-\vx^k\|_{\vx^k}^2,\]

\noindent
where the second inequality comes from the convexity of $g$ and the third from the optimality of $\vs^k$. The quantity $\|\vs^k-\vx^k\|_{\vx^k}$ is the Newton decrement $\lambda_k$, so we can write.
\[F(\vx^k+\alpha_k \vd^k) \le F(\vx^k) + \omega_* (\alpha_k \lambda_k) - \alpha_k \lambda_k^2.\]
We can interpret this result as follows: we have a local upper-bound on $F$ along the direction of $\vs^k-\vx^k$ which holds for all $\alpha_k$ such that $\alpha_k \lambda_k < 1$.  
By optimizing the upper-bound with respect to $\alpha_k$, we get that the best step is $\alpha_k^* = 1/(1+\lambda_k)$. We can check a posteriori that such $\alpha_k$ respects the condition $\alpha_k \lambda_k < 1$, and we get by direct computation
\[F(\vx^{k+1}) \le F(\vx^k) - \omega(\lambda_k),\]
where $\omega : \R \rightarrow \R_+$ is defined as $\omega(t):=t-\ln(1+t )$.

Let us clearly state the difference between Majorization-Minimization and this method. In MM we optimize directly a tight upper-bound of the loss. Here we optimize a quadratic approximation of $f$ which is not an upper-bound of the loss. The only thing that optimizing on this surrogate will guarantee is that the optimum $\vs^k$ verifies
\[\nabla f(\vx^k)^T (\vs^k-\vx^k) + g(\vs^k)-g(\vx^k) \le - \|\vs^k-\vx^k\|_{\vx^k}^2.\]
This and the self-concordance property of $f$ allows us to get an upper-bound on $F$ along the direction $\vd^k$ as
\[F(\vx^k+\alpha_k \vd^k) \le F(\vx^k) + \omega_* (\alpha_k \lambda_k) - \alpha_k \lambda_k^2.\] which guarantees descent by minimizing in the step size $\alpha_k$.

\section{The Block Successive Convex Approximation (BSCA) algorithm}
\label{sec:bsca}

In this section, we slightly modify the \ref{eq:NMF_KL} problem by ensuring that all entries of the factors are greater than a fixed $\epsilon >0$. That is, we consider the modified problem
\begin{equation}
    \label{eq:NMF_KL_eps}
    \min_{\mW \ge\epsilon, \mH \ge \epsilon} \KL(\mV \Vert \mW \mH),\tag{NMF-KL-\ensuremath{\epsilon}}
\end{equation}
with $\epsilon>0$ very small, usually taken to be machine precision. This perturbation of the problem is discussed in \cite{hien_algorithms_2021}. It is  common in the literature as it avoids the zero-locking phenomena \cite{gillis2020nonnegative}, that is the factor entries being stuck at zero and enable to change due to the multiplicative nature of the update algorithm. It also makes studying the problem easier as the objective function is now continuous and continuously differentiable on the whole search space. This is the main reason for which we restrict ourselves to this perturbation in this Section. Remark that our algorithm can easily be adapted to solve problem \ref{eq:NMF_KL_eps} by changing the maximum with $0$ at Line \ref{line:update_H} of \Cref{alg:KL-HALS} to a maximum with $\epsilon$.

We show using the results of Razaviyayn et al. \cite{razaviyayn2013unified} that with a right choice of step at Line \ref{line:KL-HALS:step_size_selection} in \Cref{alg:KL-HALS} the alternating updates of $\mW$ and $\mH$ will converge to a stationary point of \ref{eq:NMF_KL_eps}. The BSCA algorithm \cite{razaviyayn2013unified} solves the problem
\[\min_{\vx=(\vx_1,\dots,\vx_n)} f(\vx_1,\dots,\vx_n) \text{ such that } \vx_i \in \mathcal{X}_i, i=1,\dots,n,\]
where $\mathcal{X}_i \subset \R^{m_i}$ are closed convex sets, and $f: \prod_{i=1}^n \mathcal{X}_i \rightarrow \R $ is continuous. This encompasses the \ref{eq:NMF_KL_eps} problem by taking $n=2$. The algorithm cycles through each block, decreasing the loss with respect to a single block while others stay fixed. This is thus an alternating algorithm similar to ours. The $i$th block is updated by minimizing a convex approximation of the function restricted to $\vx_i$. Let $h_i(\vx,\vy)$ denote a convex approximation function for the $i$th block. Suppose that $h_i(\vx_i,\vx) $ is a first order approximation of $f(\vx)$ at each point $\vx$, in other words for all $\vx \in \prod_{i=1}^n \mathcal{X}_i$ %
\begin{equation}
\label{eq:first_order_approx}
    h_i'(\vy_i,\vx;\vd_i) \rvert_{\vy_i=\vx_i} = f'(\vx;\vd), \quad \text{for all } \vd = (0,\dots,\vd_i,\dots,0) \text{ with } \vx_i + \vd_i \in \mathcal{X}_i,
\end{equation}
where $f'(\vx;\vd)$ denotes the $\vd$ directional derivative and is defined as
\[f'(\vx;\vd) \triangleq {\lim \inf}_{\lambda \downarrow 0} \frac{f(\vv+\lambda\vd)-f(\vx)}{\lambda}.\]%
Recall that a stationary point of $f$ is an $\vx$ such that $f'(\vx;\vd)\ge0$ for all $\vd$ such that $\vx+\vd$ belongs to $\mathcal{X}$.

\begin{theorem}[\cite{razaviyayn2013unified}, Theorem 4]
\label{thm:BSCA}
    Suppose that $f$ %
    is continuously differentiable and \cref{eq:first_order_approx} holds. Furthermore, assume that $h(\vx,\vy)$ is strictly convex in $\vx$ and continuous in $(\vx,\vy)$. Then every limit point of the iterates generated by the BSCA algorithm (see \Cref{alg:BSCA}) is a stationary point of the block optimization problem.
\end{theorem}

\begin{algorithm}
    \caption{Pseudo-code for the BSCA algorithm \cite{razaviyayn2013unified}}\label{alg:BSCA}
    \hspace*{\algorithmicindent} \textbf{Inputs:} A feasible point $\vx^0$, $\sigma \in (0,1)$
    \begin{algorithmic}[1]
        
        \STATE $r:=0$
        \WHILE{convergence not reached}
            \STATE $r=r+1$, $i=(r \mod n)+1$
            \STATE Find $\vy_i^r \in \arg \min_{\vx_i \in \mathcal{X}_i} h_i(\vx_i,\vx^{r-1})$
            \STATE Set $\vd^r = (0,\dots,\vd^r,\dots,0)$
            \STATE Armijo stepsize rule: Choose $\alpha^{\text{init}}>0$ and $\beta\in(0,1)$. Let $\alpha^r$ be the largest element in $\{\alpha^{\text{init}}\beta^j\}_{j=0,1,\dots}$ satisfying:\[f(\vx^r)-f(\vx^r+\alpha^r \vd^r) \ge - \sigma \alpha^r f'(\vx^r;\vd^r)\] \label{line:bsca_armijo}
            \STATE Set $\vx^r = \vx^{r-1} + \alpha^r(\vy^r-\vx^{r-1})$
        \ENDWHILE
    \end{algorithmic}
\end{algorithm}

In \Cref{alg:KL-HALS} we use as convex approximations $h_i$ the second-order Taylor expansion of our loss function restricted to block $i$. One can check immediately that they verify \cref{eq:first_order_approx}. However, the strict convexity assumption is not always respected: in our context it is equivalent to the Hessian matrix $\mH = \mW^T \diag(\frac{\vv}{\mW \ti{\vh}}) \mW$ being positive definite. For example if $\vv$ has less than $R$ non-zero entries then the rank of $\mH$ is less than $R$ and it is not positive definite. Also if $\text{rank}(\mW) < R $ then $\text{rank}(\mH) < R $ and the assumption on surrogate functions is not respected. 
The proof of \Cref{thm:BSCA} ensures that there exists an Armijo step at Line \ref{line:bsca_armijo}.  \Cref{thm:BSCA} shows that if we replace Line \ref{line:KL-HALS:step_size_selection} with an Armijo step-size selection we will converge to a stationary point of \ref{eq:NMF_KL_eps}. %

\section{Optimal scaling with Sinkhorn's algorithm}

\label{apdx:opti_scaling}

It is known \cite{ho_non-negative_2008,pham2026second} that the optimal scaling for columns is the one that makes column sums of $\mW \mH$ equal to column sums of $\mV$. It can easily be proved by computing the first order optimality conditions of the problem
\[\min_{\mLambda \in \R_+^{N\times N} \text{diagonal}} \KL(\mV \Vert \mW \mH \mLambda) = \min_{\vlambda \in \R_+^N} \sum_{j=1}^N \sum_{i=1}^M - V_{i,j} \log(\lambda_j (\mW\mH)_{i,j}) + \lambda_j (\mW\mH)_{i,j}, \]
where $\mLambda = \diag(\vlambda)$. 
Setting the first derivative with respect to $\lambda_j$ to zero gives
\[\sum_{i=1}^M -V_{i,j} \frac{1}{\lambda_j} + (\mW\mH)_{i,j} = 0\Leftrightarrow \lambda_j = \frac{\sum_{i=1}^M V_{i,j}}{\sum_{i=1}^M (\mW\mH)_{i,j}}.\]
In other words one should make column sums of $\mW \mH$ equal to column sums of $\mV$. The same is true for row sums by transpose. Scaling rows and columns alternatively is exactly Sinkhorn's matrix scaling algorithm \cite{sinkhorn1967concerning}. A pseudo-code for Sinkhorn's algorithm applied to NMF is given in \Cref{alg:sinkhorn_nmf}

\begin{algorithm}
    \caption{Sinkhorn's algorithm for KL NMF\label{alg:sinkhorn_nmf}}
    \begin{algorithmic}[1]
        \STATE Given $\mV$, and the initial random factors $\mW$ and $\mH$
        \STATE $\vr$ := sum\_rows($\mV) \in \R_+^{M}$
        \STATE $\vc$ := sum\_columns($\mV) \in \R_+^{N}$
        \STATE $\va := \1_M$
        \STATE $\vb := \1_N$
        \FOR{$J$ iterations}
            \STATE $\va := \vr/(\mW(\mH\vb))$%
            \STATE $\vb := \vc/(\mH^T(\mW^T \va))$
        \ENDFOR
        \STATE \RETURN $\text{diag}(\va)\mW$,$\mH\text{diag}(\vb)$
    \end{algorithmic}
\end{algorithm}

Computing row and column sums of $\mV$ costs $O(MN)$ in the dense case and $O(\text{nnz}(\mV))$ in the sparse case. Each iterations computes only matrix-vector products and vector-vector divisions, for a cost of $O((M+N)R)$. The total cost of $J$ iterations is thus $O(\text{nnz}(\mV)+J(M+N)R)$ for sparse matrices or $O(MN+J(M+N)R)$ for dense matrices. Taking $J=10$ is usually enough to get row and column sums equal to the right value up to machine precision.

\section{Details on datasets}

The datasets that we use are listed in \Cref{tab:real_data_sets} along with their source, dimension, rank, and the time we let algorithms run on them. 

\begin{table}[h!]
    \centering
    \begin{tabular}{|c|c|c|c|c|}
    \hline
    dataset & size & nnz & rank & run time (seconds)\\
    \hline
    \multicolumn{5}{|c|}{Audio}\\
    \hline
    MyHeart \cite{fevotte2011majorization} & $129 \times 9312$ & dense & 10 & 5\\
    \makecell{MAPS\_MUS-bach\dots\\ \cite{emiya2010maps}} & $1000 \times 1450$ & dense & \{2,11\} & 5\\
    \makecell{MAPS\_MUS-bach\dots\\ \cite{emiya2010maps}} & $1000 \times 1450$ & dense & \{23,45\} & 10\\
    \hline
    \multicolumn{5}{|c|}{Images}\\
    \hline
    MIT-CBCL \cite{sung96} & $361 \times 2429$ & dense & 45 & 5\\ %
    ORLfaces \cite{samaria1994parameterisation} & $10304 \times 400$ & dense & 25 & 10\\
    Urban \cite{bioucas2012hyperspectral}& $162 \times 94249$ & dense & 6 & 15\\
    Frey \cite{frey_faces} & $560 \times 1965$ & dense & 25 & 15 \\
    \hline
    \multicolumn{5}{|c|}{Documents \cite{cluto}}  \\ %
    \hline
    cacmcisi & $4663 \times 14409$ & $83181$ & $2$ & 5\\
    classic & $7094 \times 41681$ & $223839$ & $4$ & 10\\
    cranmed & $2431 \times 31720$ & $140658$ & $2$ & 10\\
    fbis & $2463 \times 2000$ & $393386$ & $17$ & 30\\
    hitech & $2301 \times 22498$ & $346881$ & $6$ & 30\\
    k1a & $2340 \times 21839$ & $349792$ & $20$ & 30\\
    k1b & $2340 \times 21839$ & $349792$ & $6$ & 10\\
    la1 & $3204 \times 29714$ & $484024$ & $6$ & 20\\
    la2 & $3075 \times 19692$ & $455383$ & $6$ & 20\\
    la12 & $6279 \times 30125$ & $939407$ & $6$ & 30\\
    mm & $2521 \times 29973$ & $490062$ & $2$ & 20\\
    ohscal & $11162 \times 11465$ & $674365$ & $10$ & 30\\
    re0 & $1504 \times 2886$ & $77808$ & $13$ & 5\\
    re1 & $1657 \times 3758$ & $87328$ & $25$ & 10\\
    reviews & $4069 \times 36746$ & $781635$ & $5$ & 30\\
    sports & $8580 \times 27673$ & $1107980$ & $7$ & 30\\
    tr11 & $414 \times 6429$ & $116613$ & $9$ & 10\\
    tr12 & $313 \times 5804$ & $85640$ & $8$ & 10\\
    tr23 & $204 \times 5832$ & $78609$ & $6$ & 5\\
    tr31 & $927 \times 10128$ & $248903$ & $7$ & 10\\
    tr41 & $878 \times 7454$ & $171509$ & $10$ & 10\\
    tr45 & $690 \times 8261$ & $193605$ & $10$ & 10\\
    wap & $1560 \times 8460$ & $220482$ & $20$ & 20\\
    Verb \cite{lu2026exterior,tangherlini2016mommy} & $282 \times 1528$ & $11451$ & $\{10,20,40,80\}$& 5\\
    \hline
    \end{tabular}
    \caption{List of real datasets. Columns give in order: the name and source, the dimensions $(M,N)$, whether the dataset is dense and if not the number of non-zeros in the matrix, the rank $R$, the time we let algorithms run for each instance}
    \label{tab:real_data_sets}
\end{table}

\section{Details of experiments on synthetic datasets}

\subsection{Low-rank synthetic datasets}
\label{sm:exp_synth_lr}

\begin{figure}[h!]
    \centering
    \includegraphics[width=0.9\linewidth]{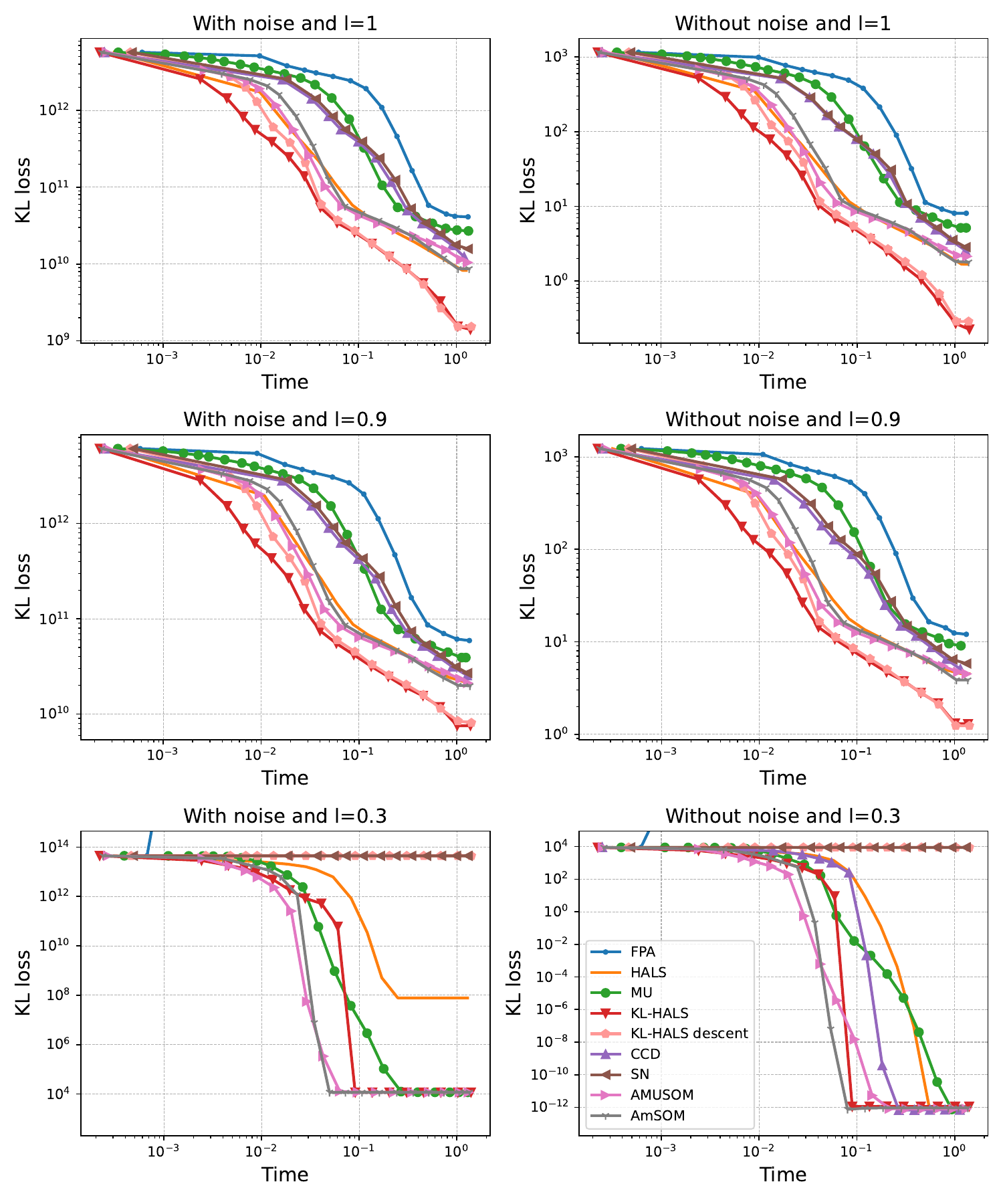}
    \caption{Median value of the KL loss on $200 \times 200$ low-rank matrices. The top row corresponds to dense factors ($l = 1$), the middle row to slightly sparse factors ($l = 0.9$), and the bottom row to very sparse factors ($l = 0.3$). The left column corresponds to noiseless matrices, the right column to noisy matrices. The curve for the CCD algorithm does not appear on the bottom left plot because it yielded \texttt{Nan} values on some of its runs.}%
    \label{fig:res_lr_200}
\end{figure}

\begin{figure}[h!]
    \centering
    \includegraphics[width=0.9\linewidth]{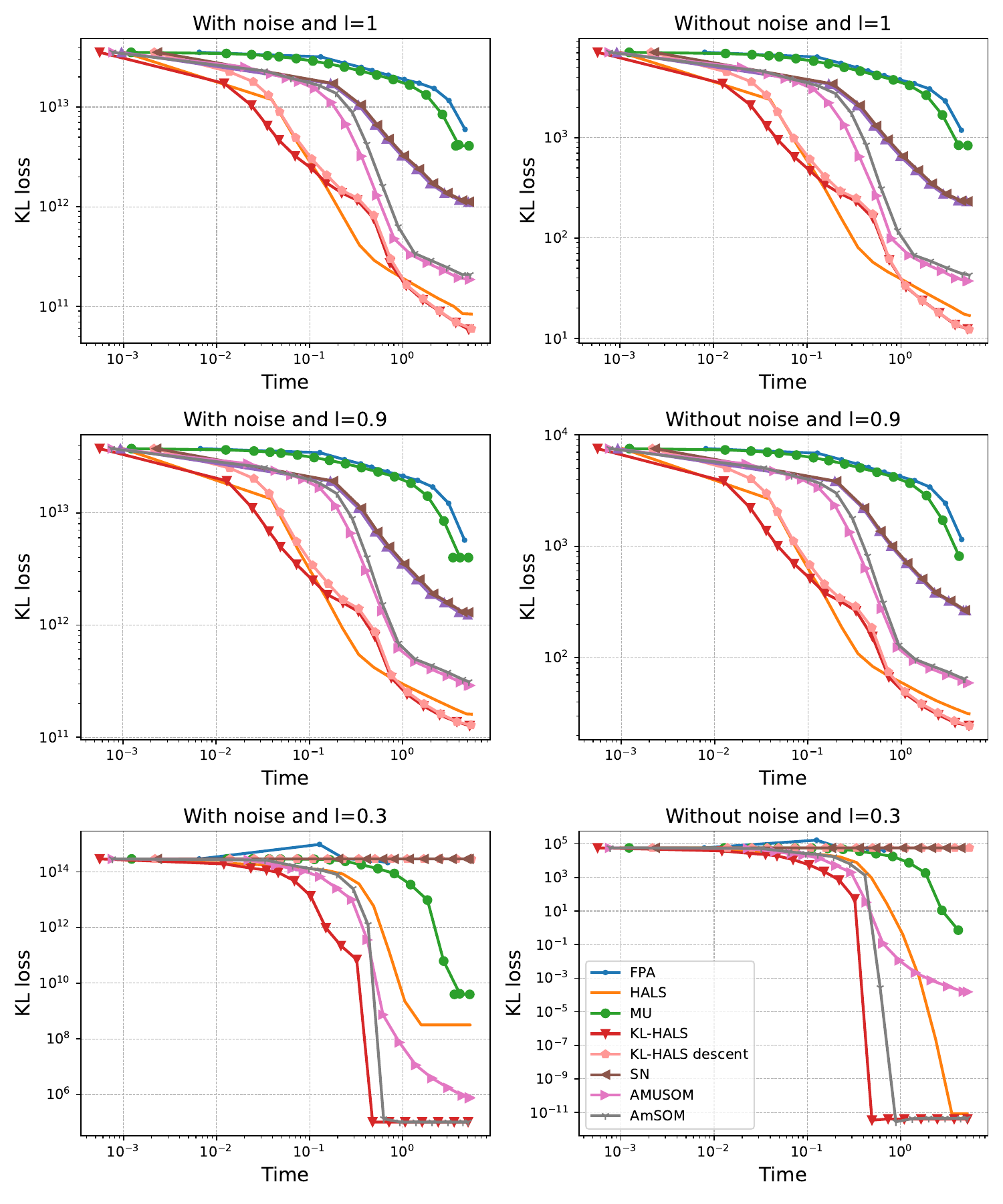}
    \caption{Median value of the KL loss on $500 \times 500$ low-rank matrices. The top row corresponds to dense factors ($l = 1$), the middle row to slightly sparse factors ($l = 0.9$), and the bottom row to very sparse factors ($l = 0.3$). The left column corresponds to noiseless matrices, the right column to noisy matrices. The curve for the CCD algorithm does not appear on the two bottom plots because it yielded \texttt{Nan} values on some of its runs.}%
    \label{fig:res_lr_500}
\end{figure}

Let us comment on the results of \Cref{fig:res_lr_200,fig:res_lr_500}:%
\begin{itemize}
    \item KL-HALS exhibits good behavior on all synthetic low-rank datasets. It converges to the lowest loss value on all instances. Only for $(M,N,R)=(200,200,10)$ it does not reach the lowest loss first, but only after AMUSOM and AmSOM.
    \item KL-HALS with the descent property follows closely on KL-HALS for $l=1$ and $l=0.9$. However for $l=0.3$ its loss decrease is too little to be seen.
    \item FPA exhibits the worse behavior for all parameters: it either has the worst convergence curve or does not converge (\Cref{fig:res_lr_200} last row).
    \item HALS is amongst the best algorithms, except for high sparsity factors ($l=0.3$) where it does not reach the best loss (for noisy data) or it reaches it later than other algorithms (noiseless case).
    \item MU is amongst the worst algorithms for factors with no or very little sparsity ($l\in\{1,0.9\}$). For high sparsity factors it converges to the best loss in the noisy case but not in the noiseless case.
    \item For $l\in\{1,0.9\}$ SN and CCD are quite close and are amongst the worst algorithms. For very sparse factors ($l=0.3$) the loss decrease of SN is not sufficient to be seen and the loss of CCD explodes on three parameter instances. For this reason the curve of CCD does not appear on the corresponding three figures. The fact that the only descending step sizes SN finds are very small could explain why CCD diverges by making large steps.
    \item AMUSOM and AmSOM follow each other closely and exhibit fast convergence on all instances. The only case AMUSOM struggles to converge to the best loss is $(M,N,R,l)=(500,500,20,0.3)$
\end{itemize}

\subsection{Full-rank synthetic datasets}

\label{sm:exp_synth_fr}

\begin{figure}[h!]
    \centering
    \includegraphics[width=0.9\linewidth]{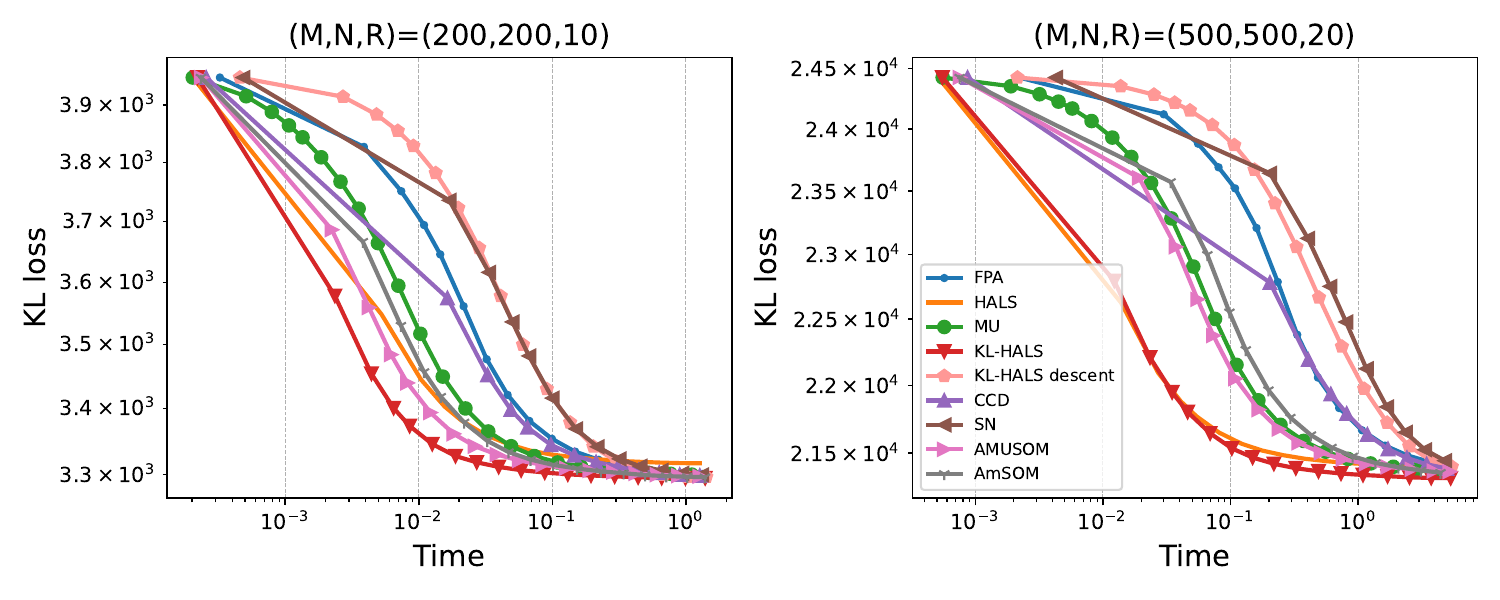}
    \caption{Median value of the KL loss on full-rank synthetic datasets: left - 200×200, right - 500×500.}
    \label{sm:fig:res_fr}
\end{figure}

Looking at the results of \cref{sm:fig:res_fr} we can observe that all algorithms converge to roughly the same loss. However their convergence speed differs. We can make the following comments on their respective convergence speed:
\begin{itemize}
    \item KL-HALS-descent and SN are the two algorithms that converge the slowest. They are the two algorithms that use the self-concordance property for choosing descending step sizes.
    \item FPA starts slower than CCD but then catches up and the two algorithms end up close to one another
    \item MU, AMUSOM and AmSOM behave roughly the same. This echos the results of Pham et al. \cite{pham2026second} who design AMUSOM and AmSOM as variations of MU.
    \item HALS converges fast for $(M,N,R)=(200,200,10)$ but to a loss slightly higher compared to other algorithms. For $(M,N,R)=(500,500,20)$ it follows KL-HALS closely with a very fast convergence speed.
    \item For both matrix sizes KL-HALS is clearly the fastest algorithm to converge.
\end{itemize}

\end{document}